%% file: main.tex
\documentclass{siamonline220329}

\input{shared}

\ifpdf
\hypersetup{
  pdftitle={Diffeomorphic Measure Matching with Kernels},
  pdfauthor={Pau Batlle, Bamdad Hosseini, Houman Owhadi, and Biraj Pandey}
}
\fi

\usepackage{enumitem}

\usepackage{amsmath,amssymb, multirow}
\usepackage{overpic}
\usepackage{xr}

\newcommand{\R}{\ensuremath\mathbb{R}}

\renewcommand{\Y}{\ensuremath\mathcal{Y}}
\renewcommand{\U}{\ensuremath\mathcal{U}}
\renewcommand{\W}{\ensuremath\mathcal{W}}

\DeclareMathOperator*{\esssup}{ess\,sup}

\newcommand{\MMD}{{\rm MMD}}

\newtheorem{maintheorem}{Main Theorem}

\begin{document}

\maketitle

% REQUIRED
\begin{abstract}

This article presents a general framework for the transport of probability measures 
towards minimum divergence generative modeling and sampling using ordinary differential equations (ODEs) and 
Reproducing Kernel Hilbert Spaces (RKHSs), inspired by ideas from diffeomorphic matching 
and image registration. A theoretical analysis of the proposed method is presented, 
giving a priori error bounds in terms of the complexity of the model, the number of samples 
in the training set, and model misspecification. An extensive suite of numerical experiments further highlights the 
properties, strengths, and weaknesses of the method and extends its applicability to 
other tasks, such as conditional simulation and inference.

\end{abstract}

% REQUIRED

% REQUIRED

\section{Introduction}
\label{sec:intro}

% This article presents a general framework for transport of probability measures 
% towards generative modeling and statistical inference based on 
% the methodology of diffeomorphic matching, 
% the theory of ordinary differential equations (ODEs), and 
% reproducing kernel Hilbert spaces (RKHSs).

Sampling/simulation of probability measures is a fundamental task in data science, computational statistics, and uncertainty quantification (UQ). 
Given a {\it target} probability measure $\nu \in \PP(\R^d)$ for some $d \ge 1$, the goal of sampling is 
to simulate a random variable that is distributed according to this target. 
Markov chain Monte Carlo (MCMC) \cite{robert1999monte, stuart2010inverse,
cotter2013mcmc, hairer2014spectral, cui2016dimension, beskos-geometric-mcmc, betancourt2017geometric, durmus2017nonasymptotic, hosseini2019two, garbuno2020interacting, hosseini2023spectral} and variational inference (VI) \cite{blei2017variational, zhang2018advances, fox2012tutorial} algorithms are  perhaps the 
most well-known families of algorithms for this task, but 
in recent years, there has been a significant increase in interest surrounding
the family of transport-based sampling algorithms \cite{marzouk2016sampling}
due to the impressive performance of generative models \cite{ng2001discriminative, jebara2012machine} such as generative adversarial networks (GANs) \cite{goodfellowGANS, goodfellow2016deep} and 
normalizing flows (NFs) \cite{rezende2015variational, kobyzev2020normalizing, papamakarios2019normalizing}.  
Broadly speaking, these transport-based algorithms find a transport map 
$T^\star$ such that $T^\star\# \eta \approx \nu$ for some reference measure 
$\eta \in \PP(\R^d)$\footnote{In practice it is customary to define 
$\eta$ on a lower dimensional latent space, but we will not consider 
this setting in this article.} that is 
easy to simulate, e.g., a standard normal. Then, samples from the target 
$\nu$ are (approximately) simulated by drawing $z \sim \eta$ 
and evaluating $T^\star(z)$. The map $T^\star$ is a minimum divergence estimator \cite{pardo2018statistical}, often found by solving  problems of the form 
\begin{equation}\label{transport-opt-generic}
    T^\star:= \argmin_{T \in \T} D(T \# \eta, \nu) + \lambda_R R(T),
\end{equation}
where $D: \PP(\R^d) \times \PP(\R^d) \to [0, +\infty]$ is a statistical divergence or loss, 
$\T$ is an appropriate, often parametric, family of transport maps 
from $\R^d$ to $\R^d$, and $R: \T \to [0, +\infty]$ is a regularization/penalty term 
with parameter $\lambda_R \ge 0$ that provides additional stability. 

In this article, we are particularly interested in the setting where $\T$ is the space of maps given by the flow of an ODE,
as popularized in 
\cite{chen2018neural, grathwohl2018ffjord}, i.e., 
\begin{equation}\label{def:transport_map}
    \T := \left\{ T : \R^d \to \R^d  \: \Big| \: T(x;v) = \phi(1, x), \quad \phi_t = v(t, \phi), \quad t \in (0,1], 
    \quad \phi(0, x) = x, \quad v \in \pmb{\Q}   \right\},
\end{equation}
where $\pmb \Q$ is a family of vector fields, often taken to be a neural network (NN) class 
in modern generative modeling 
\cite{chen2018neural, grathwohl2018ffjord, onken2021ot}. In this article, we consider a 
family of such dynamic transport maps based  on the theory of kernel methods 
and present theoretical analysis and benchmark experiments for our method. 
In the rest of this section, we give a summary of our contributions in \Cref{subsec:contributions} followed by a review of relevant literature in \Cref{subsec:lit-review}, and an outline of the article in \Cref{subsec:outline}.

\subsection{Summary of contributions}\label{subsec:contributions}
Below, we will give a concise summary of our mathematical formulation and 
main results with minimal definitions, and refer the reader to \Cref{subsec:RKHS-review}
for a review of the relevant RKHS theory.
Drawing inspiration from previous works in 
image registration and diffeomorphic matching
\cite{glaunes2004diffeomorphic,younes2010shapes, feydy2017optimal, owhadi2023ideas, de2023diffeomorphic} we
take $\pmb{\Q}$  to be a vector
valued Reproducing Kernel Hilbert Space (RKHS) \cite{kadri2016operator} 
on the set $\Gamma:= [0,1] \times \R^d$ 
and consider the following instance of \eqref{transport-opt-generic}: 
\begin{equation}
\left\{
\begin{aligned}\label{KODE-formulation-generic}
    &\minimize_{v \in \pmb \Q} && D( \phi(1, \cdot)\# \eta^N, \nu^N) + \lambda_R \| v\|_{\pmb \Q}^2,  \\
    & \text{subject to (s.t.)} && \phi_t = v(t, \phi), \quad \phi(0, x) = x, 
\end{aligned}
\right.
\end{equation}
where we replaced the reference $\eta$ and the target $\nu$ with their empirical 
approximations $\eta^N$ and $\nu^N$ obtained from $N$-i.i.d. samples.
Towards the design of a practical algorithm, we take $D$ to be the {\it maximum mean discrepancy (MMD)} defined by a  Mercer
kernel $K: \R^d \times \R^d \to \R$, i.e., it is symmetric, positive, and separately 
continuous. Following the definition \eqref{def:MMD} 
this divergence takes the simple form 
\begin{equation}\label{mmd-empirical}
\begin{aligned}
    \MMD_K^2(\phi(1, \cdot)\# \eta^N, \nu^N) & = 
    \frac{1}{N^2} \sum_{i,j=1}^N K( \phi(1, x_i), \phi(1, x_j)) 
    +\frac{1}{N^2} \sum_{i,j=1}^N K( y_i, y_j) \\  
   & \qquad - \frac{2}{N^2} \sum_{i,j=1}^N K(\phi(1, x_i), y_j).
\end{aligned}
\end{equation}
where we used $\{x_i\}_{i=1}^N$  and $\{ y_i\}_{i=1}^N$ to denote i.i.d. samples from $\eta$ and 
$\nu$ respectively.
We further take $\pmb \Q$ 
to be the RKHS of  a matrix-valued kernel $\pmb Q: \Gamma \times \Gamma \rightarrow \R^{d \times d}$ of the diagonal 
form $\pmb Q(s, s') = V(s,s') I $ where $I \in \R^{d \times d}$ is the identity matrix and $V: \Gamma \times \Gamma \to \R$ is a second Mercer kernel, independent of $K$.
Finally, we  consider a set of {\it inducing points}  
$S:= \{s_1, \dots, s_M\} \subset \Gamma$ \cite{titsias2009variational} 
and approximate \eqref{KODE-formulation-generic}
with the discrete problem 
% \begin{equation}
% \left\{
% \begin{aligned}\label{KODE-formulation-with-inducing-points}
%     &\minimize_{ c \in \R^{n \times M}} && \MMD_K( \phi(\cdot, 1)\# \eta^N, \nu^N) 
%     + \lambda_R c^T V(S, S) c,  \\
%     & \st && \phi_t = v(\phi, t), \quad \phi(x, 0) = x, \quad v(s) = \sum_{j=1}^M c^T_{:j} V(s_j, s),
% \end{aligned}
% \right.
% \end{equation}
\begin{equation}\label{KODE-formulation-with-inducing-points}
\left\{
\begin{aligned}
    &\minimize_{\{c_\ell\}_{\ell=1}^d  \subset \R^{M \times d}} && \MMD_K( \phi(1, \cdot)\# \eta^N, \nu^N) 
    + \lambda_R \sum_{\ell=1}^d c_\ell^T V(S,S) c_\ell,  \\
    & \st && \phi_t = v(t, \phi), \quad \phi(0, x) = x, \\ 
    & && v(s) = (v_1(s), \dots, v_d(s)), \quad 
    v_\ell(s) =  c_\ell^T V(S, s),
\end{aligned}
\right.
\end{equation}
where $V(S,S) \in \R^{M \times M}$ denotes the kernel matrix with entries $(V(S,S))_{ij} = V(s_i, s_j)$ and $V(S, s)$ is a (column) vector field with entries $(V(S, s))_i = V(s_i, s)$. The coefficient vectors $\{ c_\ell \}_{\ell =1}^d \subset \R^M$ denote the parameters of our model 
and are the main target of training in \eqref{KODE-formulation-with-inducing-points}. 
We also note that once the requisite kernels $K, V$, the regularization parameter $\lambda$, and the set of inducing points $S$ 
are chosen, then the above formulation can be implemented using an off-the-shelf ODE solver 
to approximate the flow $\phi(1, \cdot)$.
% Here $\pmb Q(S,S)$ is the block kernel
% matrix with entries $\pmb Q(S,S)_{ij} = \pmb Q(s_i, s_j) \in \R^{d \times d}$ for $i,j \in \{1, \dotsn M\}$, 
% $\mathbf{C}:= ( \mathbf{C}_1, \dots \mathbf{C}_M )$ is a 
% column coefficients block-vector  
% with entries $\mathbf{C}_j \in \R^d$
% $\mathbf{C}_j$ denotes the $j$-th 
% row of $\mathbf{C}$, and we denote the RKHS norm compactly as $\mathbf{C}^\intercal \pmb{Q}(S, S) \mathbf{C} := \sum_{i=1}^M \sum_{j = 1}^M \mathbf{C}_i^\intercal \pmb Q(s_i, s_j) \mathbf{C}_j$. 
Our main contributions are the theoretical analysis, the  
algorithmic development and the benchmarking of the proposed solution \eqref{KODE-formulation-with-inducing-points}
for  sampling and inference.

For our theoretical analysis
we primarily consider the setting where 
the reference $\eta$ and target $\nu$ are compactly supported and 
consider the closely related problem
\begin{equation}
\left\{
\begin{aligned}\label{KODE-formulation-with-inducing-points-related}
    &\minimize_{ \{c_\ell\}_{\ell=1}^d  \subset \R^{M}} && \MMD_K( \phi(1, \cdot)\# \eta^N, \nu^N)  \\
    & \st && \phi_t = v(t, \phi), \: \: \phi(0, x) = x, \: \:
    v(s) = (v_1(s), \dots, v_d(s)), \\ 
    & && v_\ell(s) =  c_\ell^T V(S, s), \: \:
    \sum_{\ell=1}^d c_\ell^T V(S,S) c_\ell  \le r^2,
\end{aligned}
\right.
\end{equation}
for some $r^2> 0$. The above problem is closely related to  \eqref{KODE-formulation-with-inducing-points}, indeed 
the latter can be obtained by using a Lagrange multiplier to 
relax the inequality constraints and leads to more 
efficient algorithms; see 
\Cref{subsec:regularization-bounds} for more details on how the 
two minimizers are related. 
We then obtain the following theorem 
that gives a quantitative characterization of  the approximation 
and generalization errors of the transport map obtained by solving 
\eqref{KODE-formulation-with-inducing-points-related}:

\begin{maintheorem}\label{thm:main}
Let $\Omega \subset \R^d$ be bounded and let 
$\Gamma = [0,1] \times \Omega$. 
Suppose $K: \Omega \times \Omega \to \R_{\ge 0}$ is a  Mercer kernel with RKHS $\K$  
 such that $\sup_{x, x' \in \Omega} \frac{\| K(x, \cdot) - K(x', \cdot) \|_\K}{|x - x'|} < +\infty$.
Suppose $\eta, \nu \in \PP(\Omega)$, 
and let $V: \Gamma \times \Gamma \to \R$ be a Mercer 
kernel such that $V \in C^{2k}(\Gamma \times \Gamma)$ and 
that elements of its RKHS $\mcl{V}$ vanish at the boundary of $\Omega$.
Suppose there exists a Lipschitz vector field $v^\dagger$ 
such that 
% each of its components $v^\dagger_\ell \in \pmb{\V}$ for $\ell =1, \dots, d$ and 
 $ T(\cdot; v^\dagger)\# \eta = \nu$ (recall the notation of \Cref{def:transport_map}). For some $r >0$,
let $T(\cdot; v_r^{S,N})$ be the transport map defined by  solving \eqref{KODE-formulation-with-inducing-points-related}
and define $h_S:= \sup_{s\in \Gamma} \inf_{s' \in S}  | s- s'|$ to be the
fill-distance of  $S \subset \Gamma$. 
Then, if $h_S > 0$ is sufficiently small
it holds with probability $1-\delta$, for $\delta >0$,  that
    \begin{equation*}
    \begin{aligned}
        \MMD_K( T(\cdot; v^{S, N}_r) \# \eta, \nu) 
        & \le C_2 \Bigg[ \big( \exp(C_1 r) - 1 \big) 
        h^k_S 
        + 
        \sqrt{\frac1N} \left( 2 + \sqrt{\log \left( \frac1\delta \right) } \right) \\   
        & \qquad + \frac{ \exp( L_{v^\dagger} ) - 1  }{L_{v^\dagger}} \inf_{v \in \pmb \Q_r} \| v - v^\dagger \|_\infty
        \Bigg],
        \end{aligned}
    \end{equation*}
    where $C_1, C_2 >0$ are constants  independent of $S$, $N$, and $v^\dagger$,  
    $L_{v^\dagger} := \sup_{s,s' \in \Gamma} \frac{|v^\dagger(s) - v^\dagger(s')|}{|s- s'|}$
    is the Lipschitz constant of $v^\dagger$, and 
    $\pmb \Q_r := \{ v: \Gamma \to \R^d \mid \sum_{\ell=1}^d \| v_\ell \|_{ \V}^2 \le r^2 \}$.
\end{maintheorem}
A complete proof of this theorem, including extensions and other useful auxiliary results, is given in \Cref{sec:theory}. The above result 
gives 
a complete picture of the approximation bias, 
sample complexity of our algorithm, and model misspecification error. In the context 
of generative models, and more broadly, of computational transport, 
such quantitative error bounds that account for both approximation errors 
and sample complexity are very challenging to obtain; see 
for example \cite{hutter2021minimax, divol2022optimal}. 
We note that the above bound does not take into account 
the approximation error of the ODE solver \eqref{KODE-formulation-with-inducing-points-related}
or the optimization gap in finding a global minimizer, but 
these can be accounted for by using additional terms in our 
upper bound on a case-by-case basis.

Our bound 
is interesting in  various aspects: (1) The first term, involving 
the fill-distance $h_S$, controls the approximation 
error of the model and reflects the  smoothness  of the underlying RKHS 
$\V$ and the complexity of our 
model class. 
% through the fill-distance/size of $S$. 
Indeed, this bound matches classical rates for 
scattered data approximation with kernel methods \cite{wendland2004scattered};
(2) The second term is a generalization bound for minimum 
MMD estimators following \cite{briol2019statistical} and
matches the minimax optimal approximation rate of kernel 
mean embeddings of measures \cite{sriperumbudur2016optimal, tolstikhin2017minimax}; 
(3) The third term can be viewed as a model 
misspecification error, i.e., 
choosing a kernel $V$ such that the resulting RKHS does not include the ground 
truth vector field $v^\dagger$ or that $r$ is simply too small. Thus, if the 
model is well-specified, i.e.,  $v^\dagger \in \pmb \Q_r$, then this 
term will simply vanish. Of course, there is a trade-off here since the constant 
in the first term blows up exponentially with $r$. 
Finally, we note that
while the second term in our bound is 
dimension independent, the first term is effectively dependent on 
the dimension of the problem since the fill-distance $h_S$ scales 
as $M^{-1/d}$ and so to ensure that $h_S$ is sufficiently small 
one needs an increasingly large set of inducing points, however, 
the resulting approximation error still scales as $M^{-k/d}$ 
which is acceptable when $k > d$, i.e., the class $\V$ is sufficiently smooth.

For our numerical results we solve
\eqref{KODE-formulation-with-inducing-points} using off-the-shelf 
ODE solvers and stochastic gradient descent; see \Cref{sec:algorithms}
for details. We refer to our algorithm as KODE (Kernel ODE transport)
and benchmark it on various data sets 
in low- or high-dimensional settings; on occasion, we compare KODE 
with the 
OT-Flow algorithm of \cite{onken2021ot}, which is a neural net method 
based on the dynamic formulation of the optimal transport (OT) problem. 
\Cref{fig:toy_non_auto_results} shows an example  of our 
results for a  set of 2D benchmarks of different complexity 
where we (forward) transport a standard Gaussian reference $\eta$ to a 
complicated target $\nu$. The bottom row of this figure
shows the backward transport/normalizing flow
problem of pulling $\nu$ back to $\eta$ simply by running the ODE 
backward without any training, revealing that our transport maps 
are indeed diffeomorphic. We also present  extensive 
studies on the effects of hyper-parameters and the choice of kernels 
and RKHS norms.

Finally, we also present a simple modification of our formulation 
that enables KODE to perform likelihood-free/amortized 
inference, i.e., we transport $\eta$ to $\nu$ using additional 
constraints leading to a triangular transport map which we call 
Triangular KODE (T-KODE), based on the theory of triangular transport of measures \cite{baptista2023conditional}. The resulting model is only trained once 
but it can 
identify arbitrary conditional measures of $\nu$ along pre-specified 
coordinates
as demonstrated in \Cref{fig:conditioning_time_dep}. 
Details and additional benchmarks for this method are collected in 
\Cref{sec:triangular-transport}.

\begin{figure}[ht]
\centering
\begin{overpic}[width=0.8\textwidth, trim=5 5 5 5, clip]{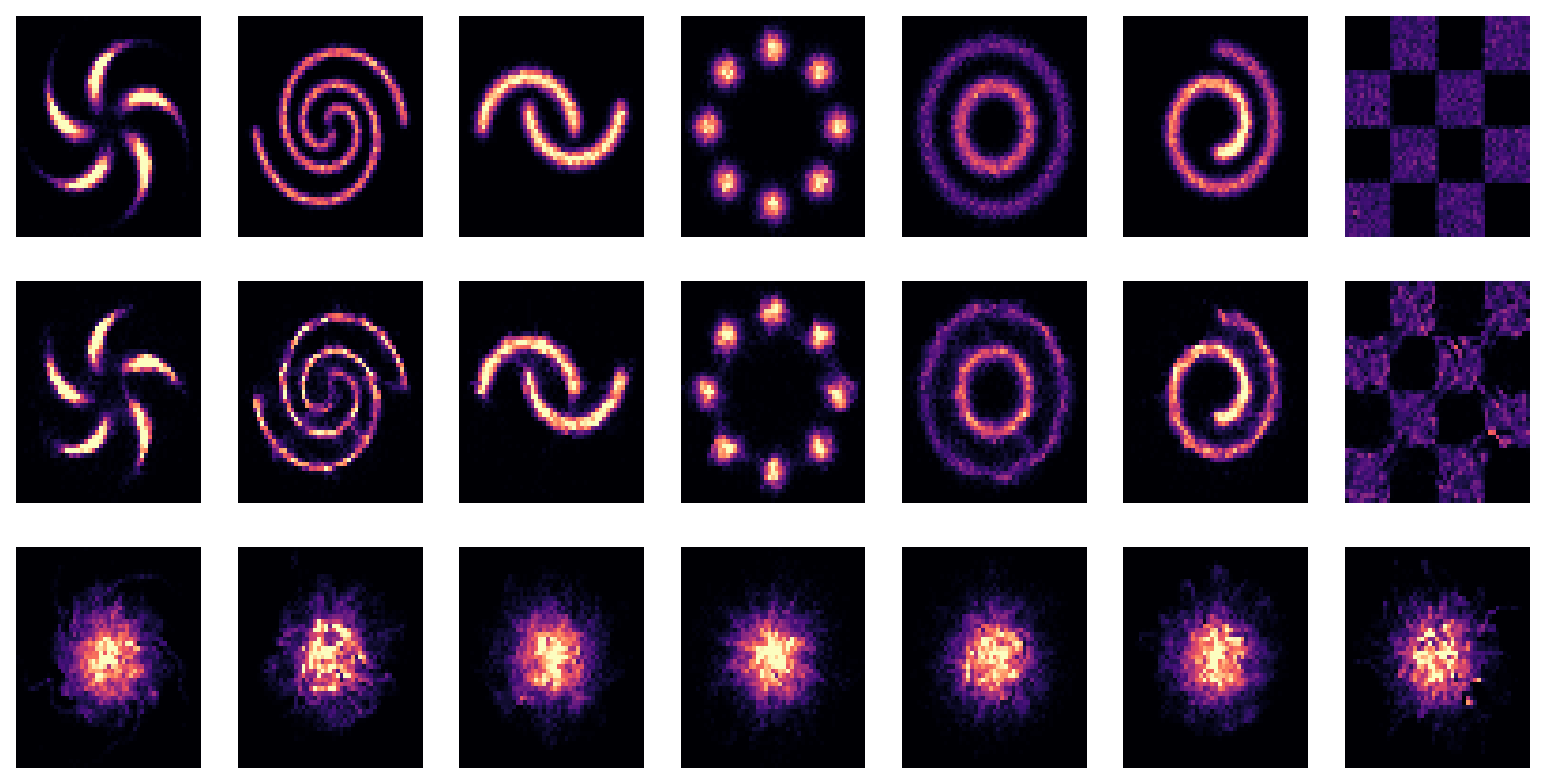}
    \put(-7, 30){\rotatebox{90}{\parbox{3cm}{\centering {Target \\ $\nu$} }}}
    \put(-7, 14){\rotatebox{90}{\parbox{3cm}{\centering {Forward \\ Transport}}}}
    \put(-7, -4){\rotatebox{90}{\parbox{3cm}{\centering {Backward \\ 
    Transport}}}}
\end{overpic}
\caption{Transport experiments on 2D benchmarks using  KODE: (Top row) The empirical samples from the target $\nu$; (Middle row) Samples generated by transporting a standard Gaussian reference $\eta$; (Bottom row) Samples generated by transporting $\nu$ backward towards $\eta$.}
\label{fig:toy_non_auto_results}
% \vskip -0.2in
\end{figure}

\begin{figure*}[ht]
% \vskip 0.5 cm
\centering
\begin{overpic}[width=.8\textwidth]{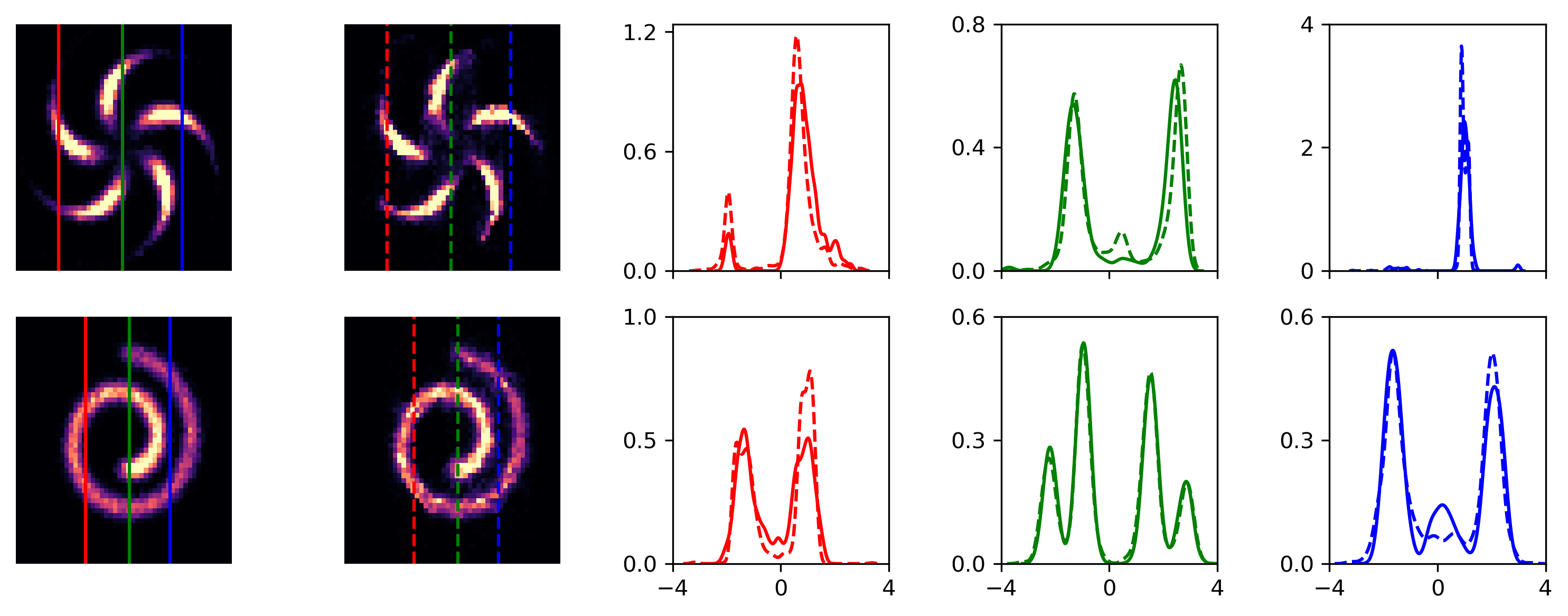}

    % first column
    \put(8, 1){{\parbox{3cm}{\tiny $y$}}}
    \put(-1, 10){\rotatebox{90}{\parbox{3cm}{\tiny $u$}}}
    \put(-1, 30){\rotatebox{90}{\parbox{3cm}{\tiny $u$}}}

    % Second column
    \put(28, 1){{\parbox{3cm}{\tiny $y$}}}  

    % remaining columns
    \put(49, -1){{\parbox{2cm}{\tiny  $u$}}}
    \put(70, -1){{\parbox{2cm}{\tiny $u$}}}
    \put(91, -1){{\parbox{2cm}{\tiny $u$}}}

    % Large titles
    \put(2, -4){{\parbox{3cm}{\footnotesize Target $\nu$}}}
    \put(17, -4){{\parbox{3cm}{\footnotesize \centering Generated samples \\ (Joint)}}}  
    \put(60, -4){{\parbox{3cm}{\footnotesize \centering Generated samples \\ (Conditional)}}}
\end{overpic}
\vskip 0.7cm
\caption{Conditional sampling using Triangular KODE for  two-dimensional benchmarks: The left-most column shows the target measure $\nu$ while the red, green, and blue lines denote slices along which conditional samples are generated. The next column shows generated samples by KODE. The remaining panels compare ground truth (solid lines) and 
generated (dashed lines) kernel density 
estimators of the requisite conditional distributions using  triangular KODE.}
\label{fig:conditioning_time_dep}
\end{figure*}
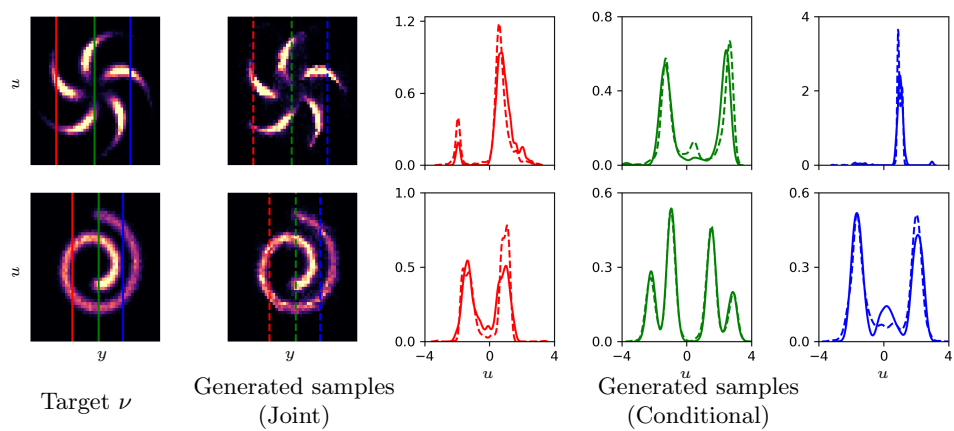

% with a particular choice of 
% the kernel $\pmb Q$ leading to vector fields that are both smooth in space and continuous
% differentiable in time. We empirically show that these are better than 
% traditional choice of $\pmb Q$ that is $L^2$ in time. We further show how 
% our formulation can be modified trivially in order to obtain a map that 
% is capable of likelihood-free conditioning.

% focus on problem 
% \eqref{KODE-formulation-generic} including a representer formula for 
% generic choices of $D$ that leads to efficient numerical strategies 
% for solving this problem. 
% We provide further theoretical results for the particular case 
% where $D$ is taken to be the MMD distance. [place holder]

% Our applied contribution is an efficient computational algorithm based on the above 
% theory by combining techniques from sparse GPs supported by extensive numerical tests 
% in the setting where $D$ is once again taken to be the MMD [place holder]

% \bp{
% We have a general formulation with a divergence metric and a regularization parameter that penalizes RKHS norms. If we use MMD as divergence and minimize integral of RKHS norm over time, we get \cite{Glaunes_2004_DiffMatch}. If we take the adversarial approach where the discriminator is the MMD, we get MmdGANs \cite{Dziugaite_2015_MMD_critic}. If we use Sinkhorn divergence as loss, we get \cite{Lara_Gonzalez-2022_DiffeoSinkhorn}.}

\subsection{Relevant Literature}\label{subsec:lit-review}

Below, we give a summary of the relevant literature to our work. 
Since the literature on generative modeling and transport is vast 
we keep the discussion focused on works that are closely 
related to ours and give a few references for broader topics.

\subsubsection{Diffeomorphic matching and learning} 

The early developments in matching measures originated from the field of diffeomorphic matching or learning, which aims to align different data sets by transforming them into a common coordinate system \cite{Zitova_Flusser_2003, beg2005computing, joshi2000landmark}. An example of an application includes aligning the images of the same scene captured from various viewpoints, sensors, and times of day.
Computational algorithms for diffeomorphic matching 
 shares a lot of similarities with our proposed method \cite{younes2010shapes, Younes_2019} as well as many recent techniques in generative modeling such as continuous NFs \cite{grathwohl2018ffjord, kobyzev2020normalizing}.
 Indeed, the problem of matching measures and distributions  has been studied extensively in \cite{bauer2015diffeomorphic, feydy2017optimal, feydy2018global} although these works were mainly focused on 
matching of images and shapes in 2D or 3D as opposed to sampling of 
high-dimensional measures.
 
 In diffeomorphic matching, flows of ODEs are utilized as diffeomorphic maps to continuously deform one image or volume into another. Notably, these diffeomorphic maps are often chosen to be flows of vector fields belonging to Reproducing Kernel Hilbert Spaces (RKHS) \cite{younes2010shapes, owhadi2023ideas, de2023diffeomorphic}, distinguishing them from current methods in machine learning that primarily use neural network parameterizations. 
 In this light, our proposed methodology also falls within 
 the category of diffeomorphic matching with a particular 
 focus towards sampling (possibly) high-dimensional 
probability measures. 

 Closely related approaches to ours in this domain include  \cite{glaunes2004diffeomorphic}, which was a major inspiration for us and 
 employs a  formulation that is a subset  of our framework. Also, the 
 theoretical analysis in that work is limited to existence and 
 consistency results as opposed to our quantitative rates.
 The recent article 
 \cite{de2023diffeomorphic}  also proposes a similar formulation to ours, with the main differences being 
 the choice of Sinkhorn divergences in place of MMD and the focus on 2D or 3D examples. 
 Another inspiration for our work is \cite{owhadi2023ideas},
which connects modern techniques, such as residual neural nets to 
the formulation of algorithms for diffeomorphic matching, although that 
article is primarily focused on supervised learning as 
opposed to generative modeling.
 Another important point of departure for us from the aforementioned 
 works  is that they utilize geodesic shooting  \cite{Younes_2019, owhadi2023ideas} to compute their transport maps as 
 opposed to direct minimization 
 in our setting.  Exploring the geodesic shooting method in our 
 framework is interesting but non-trivial, since the derivation of 
 that method  relies on the assumption that the 
 RKHS norm of $v$ is of the form 
 $\| v \|_{\pmb{\Q}}^2 = \int_0^1 \| v(t, \cdot) \|_{\pmb{\Q}'}^2 dt$
 where ${\pmb{\Q}'}$ is an RKHS on $\Omega$. Our experiments 
 in \Cref{sec:experimental_setup} suggest that imposing 
 additional smoothness in time results in better performance and 
 smoother transport maps.
 Finally, we mention the master's thesis \cite{raviola2022training} that gives an excellent 
 overview of the connection between kernel ODEs and various Neural ODE models
  from the perspective of optimal control, including applications to measure transport.

\subsubsection{ODE models in machine learning}

The past decade has seen an explosion in the research and 
development of increasingly expressive generative models, specially in 
the context of image generation, going back to the introduction 
of GANs in \cite{goodfellowGANS}. 
Since then, numerous extensions of GANs 
\cite{arjovsky2017wasserstein, gulrajani2017improved, binkowski2018demystifying, li2017mmd}
have been proposed, followed by new 
families of generative models such as NFs \cite{rezende2015variational, papamakarios2019normalizing, kobyzev2020normalizing} and
more recently diffusion models \cite{song2020score, cao2022survey}
and stochastic interpolants \cite{albergo2022building, albergo2023stochastic}. 
A core idea in the aforementioned flow models is the parameterization of 
transport maps via the compositions of parametric maps that are simple and, often, invertible. 
In the case of normalizing flows, these constraints were partially motivated 
by the use of the Kullback-Liebler (KL) divergence as a training loss 
that takes a particularly simple form during training for the normalizing flow direction, i.e., 
pulling the target $\nu$ to the reference $\eta$. Then, generating new samples from $\nu$
requires the inversion of the resulting flow in an efficient manner. This line of thinking 
led to considerable interest in continuous time dynamic models, broadly referred to as 
neural ODEs \cite{haber2020neural, chen2018neural, grathwohl2018ffjord}; simply put, a model 
akin to \eqref{KODE-formulation-generic} with $v$ parameterized via a neural net. 

The analysis of neural ODE models has attracted some attention in the literature 
although most of the existing works focus on the proof of universal approximation 
results for diffeomorphic maps \cite{li2022deep, ishikawa2023universal}
and do not provide a quantitative approximation 
rates as we do in \Cref{thm:main} and do not consider generalization 
bounds. To our knowledge, 
\cite{marzouk2023distribution} presents the most comprehensive 
analysis of neural ODEs from the perspective of transport 
problems, focusing on generalization bounds for density estimation,
as opposed to the sampling/generative modeling in our case. The results
of that work combine generalization bounds and approximation errors 
for neural ODE models, and in that sense, are very similar to our 
main theorem, however: (1) their  
analysis is limited to the normalizing flow problem with the KL loss; (2) 
it is primarily focused on neural network models; and (3) concerns 
density estimation as opposed to sampling in our case. 

We also note that there is a growing literature on the
applied analysis of transport problems as 
it pertains to generative modeling for non-ODE models, such as 
\cite{baptista2023approximation} which proposes master theorems 
for characterizing the approximation error of transport maps 
for sampling that we utilize in the proof of our main results. The articles \cite{zech2022sparse-I, zech2022sparse-II}
consider sparse polynomial approximations of triangular transport maps; 
a particularly useful construction in the context of normalizing flows.
While the aforementioned articles are primarily concerned with the 
approximation errors of parameterized transport maps, 
the articles \cite{irons2022triangular, wang2022minimax}
study the statistical consistency and sample complexity 
of such problems as it pertains to triangular maps.

\subsubsection{Connections to OT}

The theory of OT is now a mature  field 
of mathematics \cite{villani2009optimal, santambrogio2015optimal}
with deep connections to probability theory, partial differential equations,
and statistics. In recent years, there has been a growing interest 
in the computational aspects of optimal transport \cite{peyre2019computational} both in the development of 
algorithms for computing optimal maps \cite{cuturi2013sinkhorn}
as well as their applications in fields such as machine learning, data science \cite{arjovsky2017wasserstein, montesuma2021wasserstein, kuhn2019wasserstein, flamary2018wasserstein, zhuang2022wasserstein}, biology \cite{schiebinger2019optimal}, medicine \cite{gramfort2015fast},
shape analysis \cite{su2015optimal}, and economics
\cite{galichon2018optimal}. The wide adoption of 
OT in practice has also led to growing interest in the 
applied analysis of OT problems and, in particular, OT maps
\cite{hutter2021minimax, divol2022optimal, deb2021rates, del2023improved, pooladian2023minimax}. Our formulation of KODE resembles the 
dynamic formulation of OT, often known as 
the Benamou-Brenier formulation \cite[pp.~159]{villani2009optimal} which is also 
the basis of the OT-Flow algorithm \cite{onken2021ot}, but it is different 
from that problem in two important aspects: (1) our dynamics do not match 
the target measure $\nu$ exactly (we only aim to minimize the MMD) while the 
Benamou-Brenier formulation aims to push the reference $\eta$ to the 
target $\nu$ exactly; (2) we regularize our velocity fields in an RKHS 
in both space and time, while Benamou-Brenier minimizes a Bochner integral.

% \subsection{Notation}

% \begin{itemize}
% \item $\PP(\Omega)$ Borel prob. measures
%     \item pushforward
%     \item column and row vectors
%     \item For $\mu \in \PP(\Omega)$, we denote the weighted $L_\mu^p$ norm by $\| f\|_{L_\mu^p(\Omega;\Omega')} :=  \left( \int_\Omega | f|^p d\mu \right)^{1/p}$, where $|\cdot|$ is the usual Euclidean norm and $L_\mu^p(\Omega;\Omega')$ is the corresponding function space.
% \end{itemize}

\subsection{Outline}\label{subsec:outline}
The rest of the article is organized as follows: \Cref{sec:theory} contains our theoretical 
analysis of the KODE methodology and in particular, the proof of \Cref{thm:main}. Details 
of our algorithms and numerical implementations are collected in \Cref{sec:algorithms}
followed by numerical results and experiments in \Cref{sec:experimental_setup}. 
We conclude the article in \Cref{sec:conclusion} with a summary of remaining open questions and future directions.

% \subsection{Kernel methods in transport}

% \subsection{Sparse GPs}

% \subsection{Stein Variational Gradient Descent (SVGD)}

% Idea registration from Houman
%Resnet + NeuralODE + Normalizing Flows + MMD GAN
% GPs
% Optimal Transport (Gradient Flow + Stein's Variational Gradient Flow + MMD Gradient Flow (Anne Korba))
% Houman's paper (Ideas have shape)
% Conditioning (Knotk-Rosenblatt maps/Triangular maps)
% MGAN paper
% Triangular Flows
% - Optimal Transport Loss \\
% - Triangular flows \\
% - Deep ONet \\
% - Fourier Neural Networks
% - Sparse GPs
% - Random feature models

\section{Theory} \label{sec:theory}
We dedicate this section to the theoretical analysis of problem \eqref{KODE-formulation-with-inducing-points}, proving a series of auxiliary results that give
\Cref{thm:main} as a corollary but also generalize the setting of that result.

\subsection{Brief review of scalar and vector valued  RKHSs and MMD}\label{subsec:RKHS-review}

We begin with a brief review of kernel methods in the classic setting of 
real-valued kernels as well as operator/matrix-valued kernels since we need 
both for our theoretical exposition. For brevity, we will not give an exhaustive 
treatment of these topics and only review the material needed in the paper.
The interested reader may refer to \cite{berlinet-RKHS, owhadi2019operator}
for standard RKHS theory; \cite{alvarez2012kernels, kadri2016operator, owhadi2023ideas} for 
operator/matrix-valued kernels; and \cite{muandet2017kernel} for a review of kernel  
mean embeddings and MMD.

For a domain $\Omega \subseteq \R^d$ a function $K: \Omega \times \Omega \to \R$ is 
called a symmetric and positive definite kernel on $\Omega$ if $K(x,x') = K(x', x), \: 
\forall x,x' \in \Omega$ and 
for any collection of points $X:=\{x_1, \dots, x_M\} \subset \Omega$ 
the kernel matrix $K(X, X)\in \R^{M \times M}$, with entries $K(X,X)_{ij} = K(x_i, x_j)$,
is positive definite. 
We say that the kernel $K$ is {\it Mercer} if it is separately continuous (i.e., continuous 
in each of its inputs) in 
addition to being symmetric and positive definite.
Associated to the kernel $K$ is an RKHS $\K$ with 
inner product $\langle \cdot, \cdot \rangle_\K$ and norm $\| \cdot \|_\K$
 satisfying the reproducing property $f(x) = \langle f, K(x, \cdot ) \rangle_\K$
 for all $f \in \K$ and $x \in \Omega$.
% Let $\MM(\Omega)$ denote the space of (signed) Radon measures on $\Omega$. 
% We now define the RKHS of $K$ 
% $$\K:= \left\{ f(x) = \int_\Omega K(x,y) \mu_f( \dd y) \mid \mu_f \in \MM(\Omega) \text{ and } 
% \int_{\Omega \times \Omega} K(x, y) \mu_f(\dd x) \mu_f(\dd y) < +\infty
% \right\}$$
% with inner product $\langle f , g \rangle_\K 
% := \int_{\Omega \times \Omega} K(x,x') \mu_f( \dd x) \mu_g(\dd x)$
% and norm $\| f \|_\K := \sqrt{\langle f, f \rangle_\K}$. 
% The definition of the inner product above 
% further implies that the elements of $\K$ 
%  satisfy the well-known reproducing property:
% $f(x) = \langle f, K(x, \cdot \rangle_\K$ for all $x \in \Omega$.
Let us introduce the shorthand notation $K(X, x) := (K(x_1, x), \dots, K(x_M, x) ) \in \K^M$
as a {\it column vector field}, similarly $K(x, X)$ for the analogous 
{\it row vector field}. 
Then, functions of the form $f(x) = c^T K(X,x)$ for a 
(column) vector $c \in \R^M$ 
 naturally belong to $\K$ and, by the reproducing property, we have the useful identity
$\| f \|_\K^2 = c^T K(X,X) c$. For a second function 
$g(x) = (c')^T K(X', x)$, with a second point cloud $X' = \{x'_1, \dots, x'_{M'}\} \subset \Omega$ and vector of coefficients $c' \in \R^{M'}$, we also have the identity 
$\langle f, g \rangle_\K = c^T K(X, X')c'$.

% Introducing the shorthand notation $f(X) = (f(x_1), \dots, f(x_M)) \in \R^M$, we consider 
% the optimal recovery/interpolation problem, for a ground truth function $f^\dagger$: 
% \begin{equation}\label{scalar-optimal-recovery}
%     f^\ast : = \argmin_{f\in \K} \| f \|_\K \quad \st \quad  f(X) = f^\dagger(X).
% \end{equation}
% We recall that the minimizer $f^\star$ can be identified in closed form 
% using the celebrated representer formula for kernel interpolation 
% \begin{equation}\label{scalar-representer-theorem}
%     f^\ast(x) = K(x, X) K(X,X)^{-1} f^\dagger(X).
% \end{equation}

Given the Mercer kernel $K$ we also define the space $\PP_K(\Omega) \subset \PP(\Omega)$ 
of Borel probability measures $\mu$ for which $\int_\Omega \sqrt{K(x,x)} \mu(\dd x) <+\infty$
along with the kernel mean embedding $I_K: \PP_K(\Omega) \to \K$ where 
$I_K(\mu) := \int_\Omega K(x, \cdot) \mu(\dd x)$. With this notation at hand, we 
can now introduce the MMD, as a discrepancy defined on the space $\PP_K(\Omega)$:
\begin{equation}\label{def:MMD}
\begin{aligned}
        \MMD_K(\rho_1, \rho_2) & :=  \left\| I_K(\rho_1) - I_K(\rho_2) \right\|_\K \\ 
    &\equiv \Bigg( \int_{\Omega \times \Omega} K(x,x') \rho_1( \dd x) \rho_1(\dd x') 
    + \int_{\Omega \times \Omega} K(x,x') \rho_2( \dd x) \rho_2(\dd x') \\
    & \qquad - 2 \int_{\Omega \times \Omega} K(x,x') \rho_1( \dd x) \rho_2(\dd x')
    \Bigg)^{1/2},
\end{aligned}
\end{equation}
We can naturally extend the MMD to all of $\PP(\Omega)$ by setting 
its value to $\infty$ if either of the input measures do not belong to 
$\PP_K(\Omega)$. We say the kernel $K$ is {\it universal} if $\MMD_K(\rho_1, \rho_2) = 0$ 
implies that $\rho_1 = \rho_2$, in which case the MMD 
becomes a divergence in the parlance of \cite{birrell2022f}.
Finally, observe that by taking  $\rho_1, \rho_2$ in the above 
formula to be empirical measures, for example 
$\rho_1 = \frac{1}{M} \sum_{i=1}^M \delta_{x_i}$ and $\rho_2 = \frac{1}{M'} 
\sum_{i=1}^{M'} \delta_{x'_i}$ for $\{x_i \}_{i=1}^M, \{x_i' \}_{j=1}^{M'} \subset \Omega$, yields the familiar expression 
\begin{equation*}
    \begin{aligned}
        \MMD_K(\rho_1, \rho_2)  =  
        \Bigg( \frac{1}{M^2}\sum_{i,j=1}^N K(x_i, x_j)  
    + \frac{1}{{M'}^2} \sum_{i,j=1}^M K(x'_i, x'_j)  
    - \frac{2}{MM'} \sum_{i=1}^M \sum_{j=1}^{M'} K(x_i, x'_j)
    \Bigg)^{1/2},
\end{aligned}
\end{equation*}
which gives \eqref{mmd-empirical} in our introductory formulation 
of KODE.

Analogously to the case of real-valued RKHSs, we say that
a matrix/operator-valued kernel $\pmb{Q}: \Omega \times \Omega \to \R^{d \times d}$
is Mercer if it is separately continuous,  
symmetric and positive definite, i.e., 
$\pmb{Q}(x, x') = \pmb{Q}(x', x)^T$ and for any set of points $X = \{ x_1, \dots, x_M\} \subset \Omega$
and block-vector $Y = [y_1, \dots, y_M] \in \R^{d \times M}$ \footnote{One 
can think of $Y$ as a matrix, but it is more helpful in this context to 
consider it as a column vector of size $M$ whose entries  are 
in $\R^d$, i.e., a block-vector.} 
it holds that 
$  Y^T \pmb{Q}(X, X) Y :=  \sum_{i,j=1}^M y_i^T \pmb{Q}(x_i, x_j) y_j \ge 0$
where we introduced our compressed notation for the multiplication of 
block-vectors with a matrix. Every Mercer kernel $\pmb{Q}$ 
is in one-to-one correspondence with a (vector-valued) RKHS $\pmb{\Q}$ of functions 
$q: \Omega \to \R^d$ equipped with the inner product $\langle \cdot, \cdot \rangle_{\pmb\Q}$
and norm $\| \cdot \|_{\pmb\Q}$ satisfying the reproducing property: 
$\langle q, \pmb{Q}(x, \cdot) \rangle_{\pmb\Q} = q(x)$. For 
functions of the form $q(x) = \sum_{j=1}^M c_j^T \pmb Q(x_j, x)$ and 
$q'(x) = \sum_{j=1}^{M'} {c'_j}^T \pmb Q(x'_j, x)$ 
with column coefficient vectors $\{ c_j \}_{j=1}^M, \{c'_j\}_{j=1}^{M'} \subset \R^d$.
Further introduce the  block-vectors $C = [c_1,\dots, c_M]  \in 
\R^{d \times M}$ and $C'=[c'_1,\dots, c'_{M'}]  \in \R^{d \times M'}$
and point clouds $X = \{x_1, \dots, x_M\} \subset \Omega$ and $X'= \{x'_1, \dots, x'_{M'}\} 
\subset \Omega$ as before.  We then have the identity 
$\langle q, q' \rangle_{\pmb \Q} := 
C^T \pmb Q(X, X') C' = \sum_{i=1}^M \sum_{j=1}^{M'} 
c_i^T \pmb Q(x_i, x'_j) c'_j$ and subsequently $\| q \|_{\pmb \Q}^2 = C^T \pmb Q(X, X) C$.
Of particular importance to our exposition is the family of diagonal 
matrix-valued kernels of the form
$\pmb V(x, x') = V(x, x') I$ with associated RKHS $\pmb \V$  where $I \in \R^{d\times d}$ is the identity 
matrix and $V: \Omega \times \Omega \to \R$ is a scalar-valued 
Mercer kernel with RKHS $\V$.
In this case, the RKHS $\pmb\V$ can be identified as 
    $\pmb\V  = \{ v:\Omega \to \R^d \mid v_i \in \V, \quad i = 1, \dots, d \}$ 
    where we used $v_i$ to denote the $i$-th  component of $v$, i.e., 
    $v(x) = ( v_1(x), \dots, v_d(x) )$. 
    Furthermore, we have that $\| v \|_{\pmb \V}^2 = \sum_{i=1}^d \| v_i \|_\V^2$ 
    and $\langle v, v' \rangle_{\pmb\Q} = \sum_{i=1}^d \langle v_i, v'_i \rangle_\V, \: 
    \forall v,v' \in \pmb \V$.

    % Considering a ground truth, vector valued function $q^\dagger$, we can 
    % consider the analogous optimal recovery problem to \eqref{scalar-optimal-recovery}
    % for vector valued functions:
    % \begin{equation}\label{vector-valued-optimal-recovery}
    %     q^\star = \argmin_{q \in \pmb \Q} \| q \|_{\pmb \Q} \quad \st \quad 
    %     q(X) = q^\dagger(X)
    % \end{equation}
    % With our compact notation the solution can be expressed  using a similar 
    % formula to \eqref{scalar-representer-theorem}
    % \begin{equation}\label{vector-valued-representer-formula}
    %     q^\star (x) = q^\dagger(X)\pmb Q(X,X)^{-1}\pmb Q(X, x)  , 
    % \end{equation}
    % where we think of $\pmb Q(X, X)$ as an invertible operator acting 
    % on block-vectors. If we choose  $\pmb Q \equiv \pmb V$, the diagonal kernel 
    % above, then
    % problem \eqref{vector-valued-optimal-recovery} can be simplified 
    % by writing $q^\star(x) = ( q^\star_1(x), \dots, q^\star_d(x))$ 
    % where each component $q^\star_i$ is given by
    % \begin{equation}\label{component-wise-representer-formula}
    %     q^\star_i (x)= V(x, X)  V(X,X)^{-1} q^\dagger_i(X)
    %     = \argmin_{v \in \V} \| v\|_\V \quad \st \quad 
    %     v(X) = q^\dagger_i(X).
    % \end{equation}                                                                  
% \subsection{Representer theorems for transport}\label{subsec:representer-theorems}

\subsection{Error analysis}\label{subsec:error-analysis}
We dedicate this section to the proof of \Cref{thm:main}. 
We present a sequence of propositions and lemmas that are of independent 
interest and from which the proof of \Cref{thm:main} follows as a corollary.

Let $\Omega \subset \R^d$ be a bounded open set and
define the space 
\begin{equation*}
    \LL_0(\Omega; \R^d):= \left\{  f: \Omega \to \R^d \: \Big| \: \| f\|_{\LL(\Omega; \R^d)} <+\infty \quad \text{and} 
    \quad f(x) = 0 \quad 
    \forall x \in \Omega^c \right\},
\end{equation*}
 where 
$    \| f \|_{\LL(\Omega; \R^d)} := \sup_{x\in \Omega} |f(x)| 
    + \sup_{x,y \in \Omega} \frac{| f(x) - f(y) |}{|x - y|} $, that is, 
$\R^d$ valued functions on $\Omega$ that are 
bounded and Lipschitz, and vanish outside of $\Omega$. Here $| \cdot |$ denotes 
the Euclidean norm on $\R^d$.
% \bp{(What does $|f (x)|$ mean here since $f$ is vector valued? Is it a 2-norm, since we use that in \ref{prop:approximation-error-vector-field-on-S} for expanding the $\infty$-norm.)}
Further 
define $\VV := L^1( [0, 1]; \LL_0(\Omega; \R^d))$, the space of Bochner integrable 
vector fields on the (time) interval $[0,1]$,
taking values in $\LL_0(\Omega; \R^d)$, with the norm defined as  
$\| f\|_\VV = \int_0^1 \sup_{x \in \Omega} | v(t, x)|dt + \int_0^1 \sup_{x,y \in \Omega} \frac{| v(t, x) - v(t, y) |}{|x - y|} dt$.  Let us now define 
$\Gamma := [0,1] \times \Omega$ and consider a matrix-valued Mercer kernel 
 $\pmb{Q}: \Gamma \times \Gamma \to \R^{d \times d}$ with RKHS $\pmb{\mcl{Q}}$. The following lemma follows directly from classic results from ODE theory; see for example \cite[Thms.~C.3 and C.7]{younes2010shapes}:
%and 
% define the following generalization of problem \eqref{KODE-formulation-with-inducing-points-related}, by removing the restriction of the vector field 
% to the set of inducing points:
% \begin{equation}\label{KODE-formulation-idealized-restricted}
%     \left\{ 
%     \begin{aligned}
%     &\minimize_{ v \in \pmb{\Q}} && \MMD_K( \phi(\cdot, 1)\# \eta^N, \nu^N)  \\
%     & \st && \phi_t = v(\phi, t), \quad \phi(x, 0) = x,  \quad \| v \|_{\pmb\Q} \le r,
% \end{aligned}
%     \right.
% \end{equation}
% The feasibility of 

\begin{lemma}\label{lem:existence-uniqueness-homeomorphism}
    Suppose $\pmb{\mcl{Q}} \subseteq \VV$ and consider an ODE of the form 
    $\phi_t = v(t, \phi)$, with $\phi(0, x) = x$ for $v \in \pmb{\mcl{Q}}$. 
    Then (i) for all $x \in \Omega$
    there exists a unique solution $\phi(t, x)$ of the ODE on $[0,1]$ and (ii)
    the flow map associated with the ODE is at all times continuous, invertible, and 
    has a continuous inverse (i.e., a homeomorphism of $\Omega$). 
\end{lemma}

% \bh{
% Let us now 
% consider a set of inducing points $S = \{s_1, \dots, s_M\} \subset \Gamma$ as in 
% \eqref{KODE-formulation-with-inducing-points-related}, and
% define the following spaces of transport maps
% \begin{equation*}
%     \begin{aligned}
%     \T^{\pmb{\Q}} &:= \left\{ T:\Omega \to \Omega, \quad  T(x) = \phi(1, x), \quad \phi_t = v(t, \phi), 
%     \quad \phi(0, x) = x, \quad v \in \pmb{\Q}    \right\}, \\
%     \T_r^{\pmb{\Q}} &:= \left\{ T^{\pmb{\Q}} \in \T \mid \text{ where } 
%     \| v \|_{\pmb\Q} \le r \right\}, \\
%     \T_r^{\pmb{\Q}, M} &:= \left\{ T^{\pmb{\Q}} \in \T \mid \text{ where } 
%     v = \sum_{j=1}^M c_{j} \pmb Q(s_j, \cdot), 
%     \: C^T \pmb Q(S, S) C \le r, \: C = \{c_j\}_{j=1}^M \in \R^{d \times M}  \right\},  
% \end{aligned}
% \end{equation*}
% which are well-defined whenever $\pmb{\Q} \subset \VV$ by \Cref{lem:existence-uniqueness-homeomorphism}.
% }
% We now consider the transport problem 
% \begin{equation}\label{T-M-N-problem-def}
%     T^{M, N}_r := \argmin_{T \in \T_r^{\pmb{\Q}, M}}
% \end{equation}

Given a set of inducing points $S:= \{ s_1, \dots, s_M\} \subset \Gamma$ and 
a scalar $r > 0$
define the sets $\pmb \Q_r^S \subset \pmb \Q_r  \subset \pmb \Q$ 
as
\begin{equation*}
\begin{aligned}
    \pmb \Q_r & : = \left\{ v \in \pmb \Q  \: \Big| \: \| v \|_{\pmb \Q}^2 \le r^2 \right\}, \qquad
    \pmb \Q_r^S &:= \left\{ v = \sum_{j=1}^M c_j \pmb Q(s_j, \cdot) \: \:  \bigg|  \: \:
     C^T \pmb \Q(S, S) C \le r^2 \right\}.
    %, \quad C = \{c_j\}_{j=1}^M \in \R^{d \times M} \right\}
\end{aligned}
\end{equation*}
% \bp{(I have a general confusion her on the space-time kernel. We learn $c_i$ in our framework. They're the ones determining the time varying part. So, it'd be absorbed in the $Q$ here, I think. Comparing this framework to what we do, the $c_i$ are 1?)}
Let us now recall the setting of Problem~\eqref{KODE-formulation-with-inducing-points-related}
by considering a Mercer kernel $K: \Omega \times \Omega \to \R$ and
reference and target measures $\eta, \nu \in \PP(\Omega)$ with their $N$-sample 
empirical approximations $\eta^N, \nu^N$. We then define
\begin{subequations}\label{velocity-field-minimizers}
\begin{align}
v_r & :=
\left\{
\begin{aligned}
      & \argmin_{v \in \VV} &&\MMD_K ( \phi(1, \cdot) \# \eta, \nu), \\ 
      & \st &&  \phi_t  = v( t, \phi), \quad \phi(0, x) = x \: \forall x  \in \Omega,  \quad v \in \pmb \Q_r,
\end{aligned}
\right. \label{population-KODE-with-r} \\
    v^{S}_r & :=
\left\{
\begin{aligned}
      & \argmin_{v \in \VV} &&\MMD_K ( \phi(1, \cdot) \# \eta, \nu), \\ 
      & \st &&  \phi_t  = v( t, \phi), \quad \phi(0, x) = x \: \forall x  \in \Omega, \quad v \in \pmb \Q_r^S,
\end{aligned}
\right. \label{population-KODE-with-inducing-points} \\
v^{S,N}_r & :=
\left\{
\begin{aligned}
      & \argmin_{v \in \VV} &&\MMD_K ( \phi(1, \cdot) \# \eta^N, \nu^N), \\ 
      & \st &&  \phi_t  = v( t, \phi),  \quad \phi(0, x) = x \: \forall x  \in \Omega,
      \quad v \in \pmb \Q_r^S.
\end{aligned}
\right. \label{KODE-restricted-recalled}
\end{align}
\end{subequations}
noting that $v^{S,N}_r$ is precisely the velocity field that solves \eqref{KODE-formulation-with-inducing-points-related}. 
By \Cref{lem:existence-uniqueness-homeomorphism}
above, all of these problems are feasible since the ODEs for $\phi$ are 
well-defined and have unique solutions for the prescribed class of velocity fields. 
\begin{remark}
    We note that the problems in \eqref{velocity-field-minimizers} may have multiple minimizers
due to the fact that the loss functions are not convex even 
if we were to choose a minimum $\pmb \Q$ norm solution. For this reason, we simply
take $v_r, v^S_r,$ and $v^{S,N}_r$ to be any minimizer 
of these problems in the event of multiple minimizers moving forward.
\end{remark}

We now 
wish to show that the problems in \eqref{velocity-field-minimizers}  attain their minimizers under some mild assumptions
and so the velocity fields $v_r, v_r^S, v_r^{S,N}$ are well-defined. 
To this end, for a fixed $v \in \VV$ define the transport map 
\begin{equation*}
    T(x; v) := \phi(1, x), \quad \st \quad \phi_t = v( t, \phi), \quad \phi(0, x) = x.
\end{equation*}
We then have the following lemma as a consequence of classic perturbation results 
for ODEs with respect to the velocity field $v$ (see \cite[Thm.~4.7]{mattheij2002ordinary}),
stating that $T(x; v)$ is continuous in both of its arguments: 
% \bp{(Isn't this result about the stability of the flows based on IC and vector fields as opposed to continuity?)}
\begin{lemma}\label{lem:ODE-velocity-field-continuity}
    Suppose $v, v' \in \VV$ and let $L_v := \sup_{s,s' \in \Gamma} \frac{|v(s) - v(s')|}{|s  - s'|}$. Then it holds that 
    \begin{equation*}
        | T(x;v) - T(x'; v') | \le \exp(L_v) | x - x'| + \frac{(\exp( L_v) - 1)}{L_v} \| v - v'\|_\infty, 
    \end{equation*}
    where we introduced the notation $\| v - v' \|_{\infty} := \sup_{s \in \Gamma} | v - v' |$.
    % and  consequently 
    % \begin{equation*}
    %     \| T( \cdot; v) - T(\cdot; w) \|_\infty \le \exp(L_v) \; {\rm diam}(\Omega) 
    %     + \frac{(\exp( L_v) - 1)}{L_v} \| v - w\|_\infty. 
    % \end{equation*}
\end{lemma}

We also recall the following stability result for MMD  \cite[Thm.~3.2]{baptista2023approximation}:
\begin{lemma}\label{lem:MMD-stability}
    Suppose $K: \Omega \times \Omega \to \R$ is a stationary Mercer kernel 
    such that 
    \begin{equation}\label{kappa-lipschitz-at-zero}
       \sup_{x,x'\in \Omega} \frac{ \|  K(x, \cdot) - K(x', \cdot) \|_{\K}}{|x - x'|} \le L_K  
    \end{equation}
    with a constant $L_K > 0$.    
    % such that $K(x,x') = \kappa( | x - x'|)$ for a function $\kappa: \R_{\ge 0} \to \R_{\ge 0}$ 
    % satisfying 
    % \begin{equation}\label{kappa-lipschitz-at-zero}
    %     \kappa(0) - \kappa(t) \le \frac{1}{2} L^2_\kappa  t^2 , \qquad \forall t \ge 0,
    % \end{equation}
    % with a constant $L_\kappa > 0$. 
    For a measure $\eta \in \PP(\Omega)$ and
    exponent $p \in [1, +\infty]$
   consider the space $L^\infty_\eta( \Omega; \Omega):= 
    \{ T: \Omega \to \Omega \mid |  \| T \|_{L^\infty_\eta(\Omega; \Omega)} < + \infty \}$ 
    \footnote{This result can also be stated for maps in $L^p_\eta(\Omega;\Omega)$, but 
    $L^\infty$ is natural here since our maps are continuous.}
    with norm $\| T \|_{L^\infty_\eta(\Omega; \Omega)} := \esssup_{x \in \text{supp }(\eta)} |T(x)| $. 
    % \bp{(In the original paper, the assumption is Locally Lipschitz kernel. Here, we haveglobal lipschitz since we have $L_k$ instead of $L_k
    % (\Gamma, \Gamma')$. Should that be stated explicitly?)}
    Then it holds, for any pair of maps $T, T' \in L_\eta^\infty(\Omega; \Omega)$, that 
    \begin{equation*}
          \MMD_K( T\# \eta,  T'\#  \eta) 
          \le L_K \| T - T' \|_{L^\infty_\eta(\Omega; \Omega)}. 
    \end{equation*}
    % where  $p,q \in [1, +\infty]$ are H\"older exponents satisfying $1/p + 1/q = 1$.
 \end{lemma}
% \bp{(Probably stupid question but if $\eta \in \PP({\Omega})$ then isn't $\eta(\Omega) = 1$ by definition of a probability measure?)}
With the above lemmas at  hand, we can now prove the continuous dependence of the MMD 
loss function in \eqref{velocity-field-minimizers} on the underlying velocity fields.

\begin{proposition}\label{prop:MMD-K-depends-continuously-on-data}
    Suppose $v, v' \in \VV$ and $K$ is a Mercer kernel satisfying 
    \eqref{kappa-lipschitz-at-zero} with a constant $L_K > 0$. 
    Then for pairs of measures $\eta, \nu \in \PP(\Omega)$
    it holds that 
    \begin{equation*}
        | \MMD_K( T(\cdot; v) \# \eta, \nu) -  \MMD_K( T(\cdot; v') \# \eta, \nu) | 
        \le \frac{L_K (\exp( L_v) -1)}{L_v} \| v - v'\|_\infty.
    \end{equation*}
\end{proposition}

\begin{proof}
    Since $\MMD_K$ satisfies the triangle inequality and is symmetric, we have that 
    % $\MMD_K( T(\cdot; v) \# \eta, \nu) \le \MMD_K( T(\cdot; v) \# \eta, T(\cdot; v') \# \eta) + 
    % \MMD(T(\cdot; v') \# \eta, \nu)$
    % and  $\MMD_K( T(\cdot; v') \# \eta, \nu) \le \MMD_K( T(\cdot; v) \# \eta, T(\cdot; v') \# \eta) + 
    % \MMD(T(\cdot; v) \# \eta, \nu)$. Therefore 
    \begin{equation*}
        | \MMD_K( T(\cdot; v) \# \eta, \nu) -  \MMD_K( T(\cdot; v') \# \eta, \nu) | 
        \le \MMD_K( T(\cdot; v) \# \eta, T(\cdot; v') \# \eta ).
    \end{equation*}
    Then applying \Cref{lem:MMD-stability}  followed by 
    \Cref{lem:ODE-velocity-field-continuity} with $x = x'$ gives the result.
    {}
\end{proof}

As a result of the above proposition, we can now establish that 
under some regularity assumptions, the problems in  \eqref{velocity-field-minimizers}
have feasible minimizers. 
\begin{proposition}\label{prop:KODE-existence-of-minimizers}
Suppose  $K: \Omega \times \Omega \to \R$ is a stationary Mercer kernel satisfying 
    \eqref{kappa-lipschitz-at-zero}, let 
    $\eta, \nu \in \PP_K(\Omega)$, and assume $\pmb \Q$ is compactly embedded 
    in $\VV$, i.e., $\pmb \Q$ is a compact subset of $\VV$ and  
    there exists a constant $C_{\pmb \Q}> 0$ so that $\| v \|_\VV 
    \le C_{\pmb \Q} \| v \|_{\pmb \Q} $ for all $v \in \pmb \Q$. Then the minimization problems 
    in \eqref{velocity-field-minimizers} have feasible minimizers. 
\end{proposition}
% \bp{(Is it true that if $Q$ is compactly embedded then so are $Q_r$ and $Q_r^S$? Boundedness and the norm property should hold but I cannot tell if closure necessarily holds? If not, we have to explicitly mention that yeah? Cause we do use their compact embedding in the proof.)}
\begin{proof}
    Since $\eta, \nu \in \PP_K(\Omega)$ then $\MMD_K( \eta, \nu) < +\infty$ but 
    $\eta = T( \cdot; 0) \# \eta $, i.e., the velocity field arising from the 
    zero vector field. Since $0$ is an element of both spaces $\pmb \Q_r$ and $\pmb \Q_r^S $
    then all three problems are feasible. The same argument also applies to $\eta^N, \nu^N$, 
    in fact, in this case it always holds that $\eta^N, \nu^N \in \PP_K(\Omega)$.
    Then \Cref{prop:MMD-K-depends-continuously-on-data} and the assumption that $\pmb \Q_r$ (and similarly $\pmb \Q_r^S$) are 
    compact subsets of $\VV$ give the existence of minimizers in \eqref{velocity-field-minimizers}; see for example \cite[Thm.~1.9]{rockafellar2005variational}.
\end{proof}

% \bp{(In the first part, we show that there's a solution in the domain set that achieves a finite MMD. Then, Prop 2.5 shows that MMD is continuous. We assume $Q$ is compact subset of $\VV$. Then, we use the fact that a continuous function over a compact set attains a minimum. Is that reasoning correct? I'm just a little confused why you need compact embedding as opposed to just being a compact subset. Maybe as of now, we don't use it but later we do the compact embedding part?)}

We now turn our attention to controlling the error of the transport map arising 
from the velocity field $v^{S,N}_r$, which is an object that we can compute 
in practice. This velocity field is random due to its dependence on the 
empirical measures $\eta^N, \nu^N$. Therefore, we aim to obtain a high-probability 
bound on the quantity
$\MMD_K( T(\cdot; v^{S,N}_r) \# \eta, \nu)$. 
% Since $\MMD_K$ satisfies the 
% triangle inequality we have that 
% \begin{equation}\label{MMD-error-bound-triangle-inequality}
% \begin{aligned}
%         \MMD_K( T(\cdot; v^{S,N}_r) \# \eta, \nu)
%      \le & \MMD_K( T(\cdot; v^{S,N}_r) \# \eta, T(\cdot; v^{S}_r) \# \eta) =: M_1\\ 
%     & + \MMD_K( T(\cdot; v^{S}_r) \# \eta, T(\cdot; v_r) \# \eta) =: M_2 \\
%     & + \MMD_K( T(\cdot; v_r) \# \eta, \nu) =: M_3. 
%     \end{aligned}
% \end{equation}
% The first term involves the random velocity field $v^{S,N}_r$ and will lead to 
% a rate that is dependent on the number of samples $N$, essentially the sample complexity 
% of the problem. The second and third terms are no longer random, the first term 
% represent the error of approximating $v_r$ with a velocity field supported on 
% the inducing points $S$ while the third term represents the overall bias 
% of restricting our velocity fields to the ball $\pmb \Q_r$. Of course, if 
% there exists a velocity field $v^\dagger \in \pmb \Q$ such that 
% $T(\cdot ; v^\dagger) \# \eta = \nu$ then we can simply take 
% $r = \| v^\dagger \|_{\pmb \Q}$ in which case the third term will vanish. 
To achieve such a bound, we rely on \cite[Thm.~1]{briol2019statistical}, 
which gives a generalization bound for minimum $\MMD$ generative models. 
We state that result as a lemma below for convenience:
\begin{lemma}\label{lem:duncan-generalization-bound}
    Let $K: \Omega \times \Omega \to \R$ be a bounded Mercer kernel and 
    let $\mu \in \PP_K(\Omega)$. Consider a generative model (a parametric 
    probability measure) $\mu_\theta \in \PP_K(\Omega)$ parameterized 
    by $\theta \in \Theta$, taken to be an arbitrary Banach space.
    Suppose it holds that 
    \begin{enumerate}[label=(\roman*)]
        \item For every $\mu \in \PP_K(\Omega)$, there exists $C > 0$ 
        such that the set $\{ \theta \in \Theta  \mid \MMD_K( \mu_\theta, \mu) 
        \le \inf_{\theta \in \Theta} \MMD_K( \mu_\theta, \mu) + C\} $ is bounded. 

        \item For every $N >0$ and $\mu \in \PP_K(\Omega)$, there exists $C_N > 0$ 
        such that the set $\{ \theta \in \Theta  \mid \MMD_K( \mu^N_\theta, \mu) 
        \le \inf_{\theta \in \Theta} \MMD_K( \mu_\theta, \mu) + C_N\} $ is almost 
        surely bounded,
        where $\mu_\theta^N$ is an empirical approximation to $\mu_\theta$.
    \end{enumerate}
    Define $\theta^N : = \argmin_{\theta \in \Theta} \MMD_K( \mu_\theta, \mu^N)$.
    Then, with probability at least $1- \delta$ we have 
    \begin{equation*}
        \MMD_K ( \mu_{\theta^N}, \mu) \le 
        \inf_{\theta \in \Theta} \MMD_K ( \mu_\theta, \mu) 
        + 2 \sqrt{\frac2N \sup_{x \in \Omega} K(x, x)} \left( 2 + \sqrt{\log \left( \frac1\delta \right) } \right). 
    \end{equation*}
\end{lemma}

\begin{remark}
Note that in the original article \cite{briol2019statistical}, this theorem is stated 
for $\Theta$ being a finite-dimensional Euclidean space, however an inspection of the 
proof  reveals that this assumption is not needed and therefore, we 
state the result assuming $\Theta$ is a Banach space.
\end{remark}

We apply \Cref{lem:duncan-generalization-bound} with 
$\mu \leftarrow \nu$, and $\mu_\theta \leftarrow T(\cdot; v) \# \eta$ 
for $v \in \pmb \Q^S_r$, i.e., our generative models are parameterized by 
the velocity fields $v$ and so $\Theta \leftarrow \pmb \Q^S_r$. 
\begin{proposition}\label{prop:bound-on-M1}
    Suppose $K: \Omega \times \Omega \to \R$ is a Mercer kernel that 
    satisfies \Cref{lem:MMD-stability}, $\eta, \nu \in \PP_K(\Omega)$, 
    and $\pmb \Q$ is compactly embedded in $\VV$. Then with probability at least 
    $1 -\delta$, for $\delta >0$, it holds that 
    \begin{equation}\label{KODE-generalization-bound-1}
    \begin{aligned}
        \MMD_K( T(\cdot; v^{S,N}_r) \# \eta, \nu) 
        \le  & \MMD_K( T(\cdot; v^S_r) \# \eta, \nu) \\ 
        & + 
        2 \sqrt{\frac2N \sup_{x \in \Omega} K(x, x)} \left( 2 + \sqrt{\log \left( \frac1\delta \right) } \right).
    \end{aligned}
    \end{equation}
\end{proposition}
\begin{proof}
Since $\Omega$ is assumed to be bounded 
then the Lipschitz assumption on $\kappa$ implies that $K$ is bounded as well.
Further, since $\pmb \Q_r$
is readily bounded, then conditions (i, ii) of \Cref{lem:duncan-generalization-bound}
are  satisfied for our setup.  
% \bp{(The boundedness of $Q_r$ comes from the fact that $Q$ is compact and itself bounded right?)}
Applying that lemma together with
the fact that by \Cref{prop:KODE-existence-of-minimizers}
we have $\inf_{v \in \pmb \Q^S_r} \MMD_K ( T(\cdot; v) \# \eta, \nu) 
= \MMD_K ( T(\cdot; v^S_r) \# \eta, \nu)$ gives the result.
\end{proof}

We now turn our attention to the first term in \eqref{KODE-generalization-bound-1}, which 
will reflect the approximation power of the class $\pmb \Q^S_r$. 
We will bound this term under stronger smoothness assumptions on the space $\pmb \Q$
and, in particular, focus on the case of diagonal kernels.

\begin{proposition}\label{prop:approximation-error-vector-field-on-S}
    Suppose $K: \Omega \times \Omega \to \R$ is a  Mercer kernel
    that
    satisfies
    \Cref{lem:MMD-stability} and let $\eta, \nu \in \PP_K(\Omega)$. Furthermore, let $V: \Gamma \times \Gamma \to \R$
    be another Mercer kernel
    such that $V \in C^{2k}(\Gamma \times \Gamma)$ for some $k \ge 1$ with $\V$ as its RKHS.
    Let $\pmb Q(s, s') = V(s, s') I$ be the corresponding 
    diagonal, matrix-valued kernel defined from $V$ and take 
    $\pmb \Q$ to be its RKHS
    which is assumed to be compactly embedded in $\VV$. Suppose $S = \{ s_1, \dots, s_M \} \subset \Gamma$ is a 
    collection of distinct points with fill distance 
       $ h_{S} := \sup_{s \in \Gamma} \inf_{s' \in S} \| s - s' \|_2.$
    Then there exists $h_0 >0$ such that for $h_S \le h_0$ 
    it holds that 
    \begin{equation*}
        \MMD_K( T(\cdot; v_r^{S} ) \# \eta, \nu) 
        \le C L_K (\exp(C_{\pmb \Q} r) -1) h_S^k   + \MMD_K( T(\cdot; v_r) \# \eta, \nu)
    \end{equation*}
    where $C_{\pmb \Q} >0$ is the embedding constant of $\pmb \Q$ and $C >0$ is independent 
    of $h_S, r, v_r$.
% \bp{Corrected: Pls check
%     \begin{equation*}
%         \MMD_K( T(\cdot; v_r^{S} ) \# \eta, \nu) 
%         \le \dfrac{C_3 L_\kappa} {C_2} h_S^\ell (\exp(C_2 r) -1)   + \MMD_K( T(\cdot; v_r) \# \eta, \nu)
%     \end{equation*}
%     with $C_2$ as compact embedding constant and $C_3$ from diagonal kernel bound.
% }
\end{proposition}

\begin{proof}
% First observe that the Sobolev embedding theorem \cite{adams} implies that 
% $H^s(\Gamma; \R^d) \in \VV$ under the hypothesis of the theorem. Nowc
Under the hypothesis of the theorem on $K$ and $\kappa$ we have 
that $\pmb \Q \subset \VV$.
Now consider the velocity field 
\begin{equation*}
    \widehat{v}^S_r := \argmin_{v \in \VV } \| v \|_{\pmb \Q} \quad \st \quad 
    v(s_j) = v_r(s_j), \quad j=1, \dots, M, 
\end{equation*}
which is simply the interpolant of $v_r$ on  $S$ in $\pmb \Q$. 
Note  that $\widehat{v}^S_r \in \pmb \Q^S_r$ since $\| \widehat{v}^S_r \|_{\pmb \Q} 
\le \| v_r \|_{\pmb \Q}$.
Then using the optimality of  $v^S_r$  (recall \eqref{velocity-field-minimizers}) 
along with the triangle inequality for $\MMD_K$
we obtain
\begin{equation}\label{MMD-between-predicted-target}
    \MMD_K( T(\cdot; v^S_r ) \# \eta, \nu) 
    \le \MMD_K( T(\cdot; \widehat{v}^S_r ) \# \eta, T(\cdot; v_r ) \# \eta) 
    + \MMD_K( T(\cdot; v_r ) \# \eta, \nu).
\end{equation}
Let us now focus on the first term on the right-hand side.  
Applying \Cref{lem:MMD-stability} and \Cref{lem:ODE-velocity-field-continuity}
in that order 
yields 
\begin{equation*}
    \MMD_K( T(\cdot; \widehat{v}^S_r ) \# \eta, T(\cdot; v_r ) \# \eta) 
    \le L_\kappa  \frac{ \exp( L_{v_r}) -1 }{ L_{v_r}} 
    \| \widehat{v}^S_r - v_r \|_\infty.
\end{equation*}
Since we assumed $\pmb \Q$ is compactly embedded in $\VV$ then 
% \subset H^s(\Gamma; \R^d)$ then $L_{v_r} \lesssim 
% \| v_r \|_{H^s(\Gamma; \R^d)} \lesssim \| v_r \|_{\pmb \Q}$. 
% \bp{(We never assume $\Q \subset H^s$, only that $\Q$ is compactly embedded in $\VV$. But that should be enough right? We'd get 
$L_{v_r} \leq \| v_r \|_{\VV} \le C_{\pmb \Q} \| v_r \|_{\pmb \Q} \leq  C_{\pmb \Q} r$. Observing that 
 $\frac{1}{t}(\exp(t) - 1)$ is monotone increasing for $t >0$ we obtain the bound 
% there exist constant $C_1, C_2 > 0$ so that \bp{(BP: There should only be one constant $C_2$ right from compact embedding? Where does the $C_1$ constant come from again?)}
\begin{equation}\label{MMD-between-vector-field-and-interpolant}
    \MMD_K( T(\cdot; \widehat{v}^S_r ) \# \eta, T(\cdot; v_r ) \# \eta) 
    \le L_K  \frac{ \exp( C_{\pmb \Q} r) -1 }{C_{\pmb \Q} r}  
    \| \widehat{v}^S_r - v_r \|_\infty.
\end{equation}
% \bp{(Shouldnt the term in the exponential be $\exp(Cr)$? Here the constant $C$ is depends on $\Gamma$ and $C \ge 0$ right?. )}
It remains for us to bound the approximation error $\| \widehat{v}^S_r - v_r \|_\infty$. Recall that by definition 
\begin{equation}\label{vector-field-error-inf-norm-recalled}
    \| \widehat{v}^S_r - v_r \|_\infty = \sup_{s \in \Gamma} | \widehat{v}^S_r(s) - v_r(s) |
    = \sup_{s \in \Gamma} \left( \sum_{j=1}^d 
    | (\widehat{v}^S_{r,j}(s) - v_{r,j}(s) |^2 \right)^{1/2}.
\end{equation}
% \bp{(Is it true in general that for a vector valued $f: \Omega \rightarrow \R^d$,  $\| f\|_\infty = \sup_{x\in \Gamma} | f(x)| = \sup_{x\in \Gamma} \| f(x)\|_2$. I refer to this above where we define $\LL_0$.)}

On the other hand, since we assumed $\pmb Q$ is a 
diagonal kernel then $\widehat{v}^S_{r,j}$ is precisely the $\V$
interpolant  of $v_{r,j}$ for $j=1, \dots, d$. Then by 
\cite[Thm.~11.13]{wendland2004scattered}
we have that $\exists h_0 >0$ such that for $h_S \le h_0$ we
 have
\begin{equation*}
    \| \widehat{v}^S_{r,j} - v_{r,j} \|_\infty
    \le C h_S^k \| v_{r,j} \|_{\V},
\end{equation*}
for a constant $C > 0$ that is independent of $h_S$ and $v_{r,j}$.
% \bp{(Just a quick question, but in (2.11) can we move the $\sup$ through the square root so that we get the $\infty$-norm? Given that we have a positive function, I think we can?)}
Substituting this bound in \eqref{vector-field-error-inf-norm-recalled}
we obtain the error bound 
\begin{equation*}
    \| \widehat{v}^S_r - v_r \|_\infty 
    \le C h_S^k \left(\sum_{j=1}^d \| v_{r, j}\|^2_\V\right)^{1/2}
    = C  h_S^k  \| v_r\|_{\pmb\Q}
    \le Ch_S^k r
\end{equation*}
% \bp{(Here the constant $C$ is depends on $\kappa$ and $\Omega$ and $C \ge 0$ right? Combining all the constants we get that $C(\kappa, \Omega, \Gamma) > 0$ as stated?)} 
Finally, substituting this result in (\ref{MMD-between-vector-field-and-interpolant}) and plugging that back into (\ref{MMD-between-predicted-target}) yields the result.
{}
\end{proof}

Combining \Cref{prop:bound-on-M1,prop:approximation-error-vector-field-on-S}
we can finally state our main theoretical result, which serves as 
the precise version of \Cref{thm:main}: 

\begin{corollary}\label{cor:main}
    Suppose \Cref{prop:bound-on-M1,prop:approximation-error-vector-field-on-S} are satisfied and that there exists a  vector field 
    $v^\dagger \in \VV$  such that $T(\cdot; v^\dagger) \# \eta = \nu$. 
    Then  with probability $1 - \delta$, for $\delta \in (0,1)$, it holds that 
    \begin{equation}\label{master-bound}
    \begin{aligned}
        \MMD_K( T(\cdot; v^{S, N}_r) \# \eta, \nu) 
        & \le  C \Bigg[ (\exp( C_{\pmb \Q} r) - 1)  
        h^\ell_S + \frac{ \exp( L_{v^\dagger} ) - 1  }{L_{v^\dagger}} \inf_{v \in \pmb \Q_r} \| v - v^\dagger \|_\infty  \\
       & \qquad + 
        \sqrt{\frac1N} \left( 2 + \sqrt{\log \left( \frac1\delta \right) } \right) \Bigg], 
    \end{aligned}
    \end{equation}
   where   $C> 0$ 
is independent of $S, r, N$ and $  v^\dagger$.
% \bp{Corrected: Pls check.
%     \begin{equation*}
%         \MMD_K( T(\cdot; v^{S, N}_r) \# \eta, \nu) 
%         \le \dfrac{C_3 L_k}{C_2} \left[ \big( \exp(\| C_2 \cdot C_r v^\dagger \|_{\pmb\Q}) - 1 \big) 
%         h^\ell_S \right] 
%         + 
%         \sqrt{8 \kappa(0)}\sqrt{\frac1N} \left( 2 + \sqrt{\log \left( \frac1\delta \right) } \right),
%     \end{equation*}

% where $C_2$ is embedding constant, $C_3$ is bound for diagonal kernel.

\end{corollary}

\begin{proof}
    Applying \Cref{prop:bound-on-M1} and \Cref{prop:approximation-error-vector-field-on-S} 
    and combining the independent constants into $C>0$ yields the same bound as \eqref{master-bound}
    with the bias term 
    $\inf_{v \in \pmb \Q_r} \| v - v^\dagger\|_\infty$ replaced by $\MMD_K( T(\cdot ; v_r) \# \eta, \nu)$. Let $v^\dagger_r = \min_{v \in \pmb \Q_r} \| v - v^\dagger \|_\infty$ 
    and observe that the optimality of $v_r$ implies that 
    \begin{equation*}
        \MMD_K( T(\cdot ; v_r) \# \eta, \nu)  \le  \MMD_K( T(\cdot ; v^\dagger_r) \# \eta, \nu)
        = \MMD_K( T(\cdot ; v^\dagger_r) \# \eta, T(\cdot ; v^\dagger)\# \eta).
    \end{equation*}
    Applying \Cref{lem:MMD-stability} and \Cref{lem:ODE-velocity-field-continuity} 
    further gives the sequence of bounds  
    \begin{equation*}
        \begin{aligned}
    \MMD_K( T(\cdot ; v^\dagger_r) \# \eta, T(\cdot ; v^\dagger)) 
    & \le L_K \| T(\cdot ; v^\dagger_r) - T(\cdot ; v^\dagger) \|_{L^\infty_\eta(\Omega;\Omega)} \\
    & \le L_K \frac{(\exp( L_{v^\dagger}) - 1}{ L_{v^\dagger}} \| v^\dagger_r - v^\dagger \|_\infty. 
        \end{aligned}
    \end{equation*}
    which yields the desired result.
\end{proof}

\subsection{Unconstrained minimization with a regularization term}\label{subsec:regularization-bounds}
So far, our theoretical analysis has been focused on the error analysis of problem \eqref{KODE-restricted-recalled} (which also coincides with \eqref{KODE-formulation-with-inducing-points-related}), however 
our numerical algorithms are based on the unconstrained relaxation \eqref{KODE-formulation-with-inducing-points}, which we recall as 
\begin{equation}
v^{S,N}_\lambda  :=
\left\{
\begin{aligned}
      & \argmin_{v \in \VV} &&\MMD_K ( \phi(1, \cdot) \# \eta^N, \nu^N) + \lambda  \sum_{\ell=1}^d 
      \| v_\ell \|_{\V}^2 \\ 
      & \st &&  \phi_t  = v( t, \phi),  \quad \phi(0, x) = x \quad \forall x  \in \Omega, \\
      &  && v(s) = (v_1(s), \dots, v_d(s)) \quad v_\ell = c_\ell^TV(S, \cdot).
\end{aligned}
\right. \label{KODE-unconstrained}
\end{equation}
Note our abuse of notation here with $v^{S,N}_r$ denoting the solution to 
\eqref{KODE-restricted-recalled} while 
$v^{S,N}_\lambda$ denotes its unconstrained relaxation in
\eqref{KODE-unconstrained}.
 Since $v^{S,N}_\lambda$ is optimal we 
immediately obtain the bound 
\begin{equation*}
\begin{aligned}
        \MMD_K( T(\cdot; v^{S,N}_\lambda) \# \eta^N, \nu^N) 
    & \le \MMD_K( T(\cdot; v^{S,N}_r \# \eta^N, \nu^N)  \\ 
    & \qquad + 
    \lambda \max \left\{0, \sum_{\ell=1}^d \| v^{S,N}_{r, \ell}\|^2_{\V} - 
    \| v^{S,N}_{\lambda, \ell} \|_{\V}^2  \right\}.
\end{aligned}
\end{equation*}
This suggests that the bound in \eqref{master-bound} can be extended to 
$v^{S,N}_\lambda$ up to a (possibly non-zero) bias term 
concerning the choice of $\lambda$ and $r$.

\section{Numerical Implementation}\label{sec:algorithms}
In this section, we collect details around the numerical implementation of problem \eqref{KODE-formulation-with-inducing-points} and discuss our strategies in preparation for the
benchmark examples in \Cref{sec:experimental_setup}.

\subsection{Summary of the algorithm}
Noting that \eqref{KODE-formulation-with-inducing-points} can readily be 
implemented by discretizing the ODE, we focus our attention here on the choice of 
kernels and our approach for tuning hyper-parameters. 

\paragraph{The choice of the kernel $V$}
We will work with kernels $V(s,s')$ that are of product form in space and time, that is, 
\begin{equation*}
    V(s,s') \equiv V((t, x), (t',x')) = W(t, t')U(x,x')  
\end{equation*}
for Mercer kernels $W: [0,1] \times [0,1] \to \R$ and $U:\Omega \times \Omega \to \R$
with RKHS spaces $ \W$ and $ \U$, respectively.
Choosing $\W = H^1([0,1])$ yields the particularly useful form 
\begin{equation}\label{sobolev-in-time-RKHS-norm}
    \| v_\ell \|_{\V}^2 = \lambda_1 \int_0^1 \| v_\ell(t, \cdot) \|_{\U}^2 \: \dd t 
    + \lambda_2 \int_0^1 \| \dot{v}_\ell(t, \cdot) \|_{\U} ^2 \: \dd t, \qquad \ell= 1, \dots, d,
\end{equation}
where $\dot{v}_\ell$ denotes the time derivative of $v_\ell$ and  $\lambda_1, \lambda_2 >0$ are constants. We note that 
the second term is non-standard in the context of diffeomorphic matching \cite[Ch.~10]{younes2010shapes}
or OT \cite{onken2021ot} (see also \cite[pp.~159]{villani2009optimal}). In our numerical 
experiments in \Cref{sec:experimental_setup} we will present ablation studies 
that demonstrate the effect of the second term leading to smoother flow maps; see 
\Cref{fig:toy_non_auto_trajectory}. 
We will also consider settings where $\W$ consists of constant functions (i.e., $\dot{v}_\ell$
is maximally penalized) in which case the velocity field is constant in time, leading 
to an autonomous ODE in \eqref{KODE-formulation-with-inducing-points}.
We will refer to this setting as {\it Autonomous KODE}.
For the choice of the kernel $U$ we often use standard choices such as the 
Gaussian/squared exponential kernel or the Laplace kernel 
\begin{equation}\label{Gaussian-and-Laplace-kernel}
    U_{\text{Gaussian}}(x, x') = \exp\left( - \frac{ |x - x'|^2}{2 \gamma_U^2} \right)
    \qquad U_{\text{Laplace}}(x, x') = \exp\left( - \frac{ |x - x'|}{\gamma_U} \right)
\end{equation}
where in both cases $\gamma >0$ denotes a lengthscale parameter. 

\paragraph{The choice of the inducing points $S$}
Since the kernel $V$ is of product form, it is natural for us to also choose $S$ to be 
of a similar structure. To this end, we choose a set of spatial
inducing points $X = \{x_1, \dots, x_J \} \subset \Omega$ 
and writing $v_\ell (t, x) = c_j(t)^T U(X, x)$ for a set of 
coefficient vector fields $c_j: [0,1] \to \R^J$. This 
leads to a spatial discretization 
of \eqref{sobolev-in-time-RKHS-norm} in the following form:
\begin{equation*}
        \| v_\ell \|_{\V}^2 \approx 
        \lambda_1\int_0^1 c_\ell(t)^T  U(X, X) c_\ell(t) \: \dd t 
    + \lambda_2 \int_0^1  \dot{c}_\ell(t)^T U(X,X) \dot{c}_\ell(t) \: \dd t, \qquad \ell= 1, \dots, d.
\end{equation*}
The coefficient functions $c_\ell(t): [0,1] \to \R^J$ can be viewed as the trajectories of a 
a system of coupled ODEs. We further discretize the above integrals using the mid-point rule 
with intervals of size $\Delta t$ (chosen so that $1/\Delta t$ is an integer), to obtain the discrete RKHS penalty, which is implemented 
in our code
\begin{equation*}
        \| v_\ell \|_{\V}^2 \approx 
        \Delta t \sum_{k=1}^{1/ \Delta t} \left( \lambda_1 c_{\ell,k}^T  U(X, X) c_{\ell, k} 
    +\lambda_2 \dot{c}_{\ell,k}^T U(X,X) \dot{c}_{\ell,k} \right), \qquad \ell= 1, \dots, d.
\end{equation*}
The time derivatives $\dot{c}_{\ell, j}$ are further computed using 
standard finite-difference formulae such as centered differences in the interior of the 
unit interval and one-sided formulae at the boundaries. 
The aforementioned discretization of the RKHS penalty arises from choosing a 
set of space-time inducing points $S$ on a lattice obtained by 
tensorizing $X$ with a uniform grid of points in time. More precisely, let $t_k = k \Delta t$
for $k=0, \dots, 1/\Delta t$. Then, defining $s_{j,k} = (x_j, t_k) $ we 
obtain the set of inducing points $S = \{ s_{j,k} \}_{j=1, k=1}^{j=J, k = 1/\Delta t}$
that is consistent with our error analysis in \Cref{sec:theory}.

\paragraph{The choice of the kernel $K$}
In all of our experiments, we take the MMD kernel $K$ to belong to the radial family, and in 
particular, we take $K$ to be the Laplace kernel as in \eqref{Gaussian-and-Laplace-kernel}.
We note that the choice of the Gaussian kernel or the well-known 
inverse quadratic kernel is also possible, although we found that the differences were minimal.
It is important to note that in our framework, the choice of the kernels $K$ and $U$ are 
not related, and the parameters of these kernels can be tuned independently of each other.

\paragraph{Tuning hyper-parameters}
Our model contains a number of hyper-parmaeters such as the regularization parameter $\lambda$, 
the lengthscales $\gamma_U$ and $\gamma_K$ for the kernels $U, K$ respectively, and the 
step size $\Delta t$. We utilized standard cross-validation techniques and the median heuristic for choosing our kernel lengthscales 
\cite{Gretton_2012}; details can be 
found in our repository \footnote{\texttt{\url{https://github.com/TADSGroup/KernelODETransport}}}.

\section{Experiments}
\label{sec:experimental_setup}

Below, we collect the results of our numerical experiments on a collection of benchmark problems.
In all of our examples, we take the reference measure $\eta$ to be a standard Gaussian distribution.

\subsection{Overview of the experiments}
First, we evaluated the performance of KODE for sampling two-dimensional measures that are standard benchmarks for generative modeling tasks \cite{grathwohl2018ffjord}.
In these seven examples, the target measures $\nu$ are concentrated on complex shapes like a pinwheel or a checkerboard. Our data sets for these experiments consist of 25,000 samples drawn from $\nu$ 
that are split into training, validation, and test samples. We used 5000 training samples 
for all of our examples except for the checkerboard data set, for which 10000 samples were used 
(the latter is the most challenging of the 2D benchmarks).
% We used $10,000$ training examples for the checkerboard dataset and $5000$ training points for the remaining datasets.
The 2D benchmark results are collected in \Cref{sec:2D-benchmark}. 

Next, we applied KODE to higher-dimensional benchmark data sets: \texttt{POWER}, \texttt{GAS}, 
\texttt{HEPMASS}, and  \texttt{MINIBOONE} from the University of California Irvine (UCI) machine learning data repository and the Berkeley Segmentation Dataset (\texttt{BSDS300}) from UC Berkeley Computer Vision group; all of 
these  are commonly used benchmarks in the normalizing flow literature \cite{kobyzev2020normalizing, Paramakarios_autoreg}. These data sets range from 6 to 63 dimensions and have distributions 
with varying levels of complexity.
All of these data sets were implemented using their off-the-bench training and 
testing splits, and the results are collected in 
\Cref{sec:high-dim-benchmark}

% In all of our benchmarks we tune hyper-parameters using cross validation on the pertinent 
% validation sets;
% \Cref{tab:hyper} show these hyperparameters for all data sets. 
% \bhtodo{Please fix reference}

We also performed two experiments in addition to the standard benchmarks above. In \Cref{sec:MNIST}
we present an example of image generation for the MNIST data set by augmenting KODE with 
neural net features that are lightly trained. In \Cref{sec:triangular-transport} we 
present use a small modification of KODE to obtain a triangular map that is capable of 
likelihood-free inference and use it to infer the parameters of a Lotka-Volterra ODE.

% \texttt{POWER} describes the electricity power consumption from a single household located in France over a period of 47 months. \texttt{GAS} contains the readings from chemical sensors exposed to a mixture of ethylene and carbon monoxide mixtures over a 12 hour period. \texttt{HEPMASS} describes particle collisions in high energy physics while \texttt{MINIBOONE} measures neutrino oscillations from the MiniBooNE experiment at Fermilab. \texttt{BSDS300} contains $8 \times 8$ monochrome patches from the natural images. 
% For additional descriptions and pre-processing details for individual datasets, please refer to . We also train the model for generating MNIST digits. For the task of conditioning, we use the pinwheel and spiral datasets from the two-dimensional benchmarks. We also infer the parameters in a system of non-linear coupled ODEs, the Lotka-Volterra population model, based on noisy observations of the ODE trajectory. 
% The higher-dimensional benchmark data sets are pre-split into training, validation, and test data sets.
% We use the same training set up as \cite{baptista2023conditional} for the Lotka-Volterra conditioning task.
% For all the models, we tune the hyperparameters on the validation set.

% \paragraph{Training details} 
% For the 2D data sets, 

\paragraph{Comparison metrics}
We compared the performance of   KODE with the OT-Flow algorithm of 
\cite{onken2021ot}, which is also a method based on the dynamic formulation of 
transport although it utilizes neural nets to parameterize the velocity fields.
To be fair, in comparison of the time and complexity of the models, we 
trained OT-Flow using code from the original article on the same hardware 
that KODE was trained on.

In order to compare the quality of the produced samples for both methods we used 
the MMD metric but we made sure the kernel for this metric is different from 
the one used in the training of KODE to be fair to OT-Flow. 
Indeed, following \cite{onken2021ot}
we used the Gaussian kernel with a unit lengthscale for this purpose while KODE was 
trained using the Laplace kernel.
Finally, in all of our experiments we report the normalized MMD values 
which denotes the MMD between the generated samples and the test data set 
normalized by the MMD between reference samples and the test data.

% method on sampling from benchmark distributions. We compare the methods based on their number of model parameters, time taken for training, and quality of the generated samples. Sample quality is evaluated using normalized maximum mean discrepancy (MMD), which measures how much a generative algorithm reduces the MMD between reference and target samples after training. Specifically, we calculate the ratio of the MMD between generated samples and target samples versus reference and target samples, using an RBF kernel $k(x_i, x_j) = \exp{\left( - \|x_i - x_j \|_2^2 \right)}.$ This kernel is consistent with the implementation in OT-Flow \cite{Onken_2021}.

% An important point to note is that OT-Flow and KODE minimize different losses during training. OT-Flow minimizes KL-divergence loss while KODE minimizes an MMD loss. Then, it is natural to ask whether MMD is a fair metric to compare these models. 
% However, laplace kernel is used during training while the RBF kernel is used for comparison. Since these kernels belong to different families, we believe the comparison is fair and impartial. 

\subsection{2D benchmarks}\label{sec:2D-benchmark}
We start by comparing autonomous and non-autonomous KODE with OT-Flow on benchmark examples in two dimensions. \Cref{tab:numerical_results} compares these two versions of KODE 
with OT-Flow in terms of the number of parameters, training wall-clock time, 
and normalized MMD of the generated samples. Clearly, the autonomous KODE has 
the lowest number of parameters since it does not require time discretization
of the coefficient functions $c_{\ell}(t)$.
We observed that   KODE performed on par with OT-Flow on all benchmarks. The 
case of the \textit{checkerboard} measure is particularly interesting since KODE 
appears to outperform OT-Flow by a large margin. In all examples, we tried to 
match the number of degrees of freedom in non-autonomous KODE and OT-Flow, but 
we observed that the training wall-clock time for KODE was generally 
significantly shorter than OT-Flow. The autonomous KODE algorithm is 
much faster to train as expected, although this efficiency comes at the cost of test 
performance except for the {\it circles} data set where autonomous KODE appears to 
achieve the best performance despite being much simpler.

% We observe that KODE is more computationally efficient compared to OT-Flow; it achieves comparable MMD in almost half the training time with fewer parameters. Between the KODE models, the non-autonomous KODE generally outperforms autonomous KODE.

\Cref{fig:toy_non_auto_results} shows generated samples from non-autonomous KODE. The top row shows the target measure $\nu$, the middle row shows the KODE samples, and the bottom row shows the samples generated by pulling the target samples to the Gaussian reference 
by running the KODE model backward in time after training. 
In all examples, there is a good match between the target measure and the pushforward measure
and the pullback (i.e., reverse transport)  recovers the Gaussian reference measure 
with good quality.  

We further tested the effect of the time derivative penalties in our regularization 
term \eqref{sobolev-in-time-RKHS-norm} for the velocity fields. In 
\Cref{fig:toy_non_auto_trajectory}  we show the trajectories of a set of samples
from the reference to the target generated by KODE with and without penalizing the 
time derivatives, i.e., $\lambda_2 \neq 0$ or $\lambda_2 = 0$. 
We clearly observe that the sample trajectories 
are smoother in the latter case, implying that the resulting velocity fields 
are smoother and can be simulated more efficiently using adaptive ODE solvers.

% show the trajectories taken by individual samples from the reference measure to the final pushforward measure for four benchmark examples. The top row shows trajectories with the original Sobolev RKHS norm, $\| v\|^2_{\pmb \Q} = \int_0^1 \| v(t, \cdot) \|^2_{\pmb \V} dt +  \int_0^1 \| \dot{v}(t, \cdot)\|^2_{\pmb \V} dt$ while the bottom row shows trajectories with $L_2$ RKHS norm $\| v\|^2_{\pmb \Q} =  \int_0^1 \| v(t, \cdot) \|^2_{\pmb \V} dt$. In the latter, the time-derivative of $v(t, .)$ is not penalized. As expected, the former transports samples via smoother trajectories since the vector fields are constrained to be smoother in time.

\begin{table}[tbhp]
\centering
\scriptsize
\caption{Summary of numerical results on benchmark examples: we report the normalized MMD, number of trainable model parameters, and the total training time. We compare autonomous and non-autonomous KODE models with the OT-Flow model.} 
  \label{tab:numerical_results}
\begin{center}
\resizebox{\textwidth}{!}{%
  \begin{tabular}{|c|l c c c|} \hline
    
   Dataset & Model & $\#$ Parameters  & Training time(s) & Normalized MMD \\ \hline
   \multirow{3}{*}{Pinwheel} & KODE (autonomous)  & 800 & 470 & 4.0e-3 \\
   & KODE (non-autonomous) & 1000 & 476 & 3.4e-3 \\
   & \textbf{OT-FLOW} & \textbf{1229} & \textbf{910} & \textbf{1.8e-3} \\ \hline
   \multirow{2}{*}{2spirals} & KODE (auto) & 400 & 226 & 1.0e-2 \\
   & KODE (non-auto) & 1000 & 473 & 7.1e-3 \\
   & \textbf{OT-FLOW} & \textbf{1229} & \textbf{814} & \textbf{6.2e-3} \\ \hline
   \multirow{2}{*}{moons} & KODE (auto) & 200 & 220 & 9.8e-3 \\
   & \textbf{KODE (non-auto)} & \textbf{900} & \textbf{427} & \textbf{4.1e-3} \\
   & OT-FLOW & 1229 & 924 & 5.7e-3 \\ \hline 
   \multirow{2}{*}{8gaussians} & KODE (auto) & 800 & 473 & 1.1e-3 \\
   & KODE (non-auto) & 900 & 661 & 9.5e-4 \\
   & \textbf{OT-FLOW} & \textbf{1229} & \textbf{893} & \textbf{4.0e-4} \\ \hline
   \multirow{2}{*}{circles} & \textbf{KODE (auto)} & \textbf{800} & \textbf{270} & \textbf{3.4e-3} \\
   & KODE (non-auto) & 1000 & 485 & 4.5e-3 \\
   & OT-FLOW & 1229 &  825 & 7.2e-3 \\ \hline
   \multirow{2}{*}{swissroll} & KODE (auto) & 800 & 271 &  5.3e-3\\
   & KODE (non-auto) & 1000 & 485 & 4.6e-3 \\
   & \textbf{OT-FLOW} & \textbf{1229} & \textbf{899} & \textbf{3.0e-3} \\ \hline
   \multirow{2}{*}{checkerboard} & KODE (auto) & 1000 & 463 &  1.7e-3 \\
   & \textbf{KODE (non-auto)} & \textbf{1200} & \textbf{974} & \textbf{7.2e-4} \\
   & OT-FLOW & 1229 & 5919 & 1.7e-3 \\ \hline 

    % Higher Dimensional exampls
    \hline
    \multirow{3}{*}{\texttt{POWER} (d=6)} & KODE (auto)  & 6K & 0.24 &  4.8e-3 \\
   & \textbf{KODE (non-auto)}  & \textbf{12K} & \textbf{0.74} & \textbf{9.5e-4} \\
   & OT-FLOW & 18K & 0.55 &  2.0e-3 \\ \hline
   \multirow{3}{*}{\texttt{GAS} (d=8)} & KODE (auto) & 8K & 0.16& 6.1e-3 \\
    & \textbf{KODE (non-auto)}  & \textbf{64K} & \textbf{1.33} & \textbf{ 2.3e-3} \\
   & OT-FLOW & 127K & 2.36 &  3.2e-3 \\ \hline
   \multirow{3}{*}{\texttt{HEPMASS} (d=21)} & KODE (auto) & 105K & 1.39 & 5.1e-1 \\
    & KODE (non-auto)  & 84K & 1.5 &  7.0e-1 \\
   & \textbf{OT-FLOW }& \textbf{72K} & \textbf{3.45} &  \textbf{2.6e-2} \\ \hline
   \multirow{3}{*}{\texttt{MINIBOONE} (d=43)} & KODE (auto)  & 65K & 0.15 & 3.7e-1 \\
    & KODE (non-auto)  & 86K & 0.29 &  2.9e-1 \\
   & \textbf{OT-FLOW} & \textbf{78K} & \textbf{0.46} &  \textbf{3.6e-2} \\ \hline
   \multirow{3}{*}{\texttt{BSDS300} (d=63)} & KODE (auto)  & 63K & 2.89 &  2.7e-2 \\
   & KODE (non-auto)  & 252K & 3.37 &  1.8e-2 \\
   & \textbf{OT-FLOW} & \textbf{297K} & \textbf{4.09} &  \textbf{1.0e-2} \\ \hline
   \end{tabular}
   }
\end{center}
\end{table}

% % Results of non-auto transport
% \begin{figure}[ht]
% \centering
% \begin{overpic}[width=0.85\textwidth, trim=5 5 5 5, clip]{figures/toy/time_dep_results_2d.png}
%     \put(-7, 30){\rotatebox{90}{\parbox{3cm}{\centering {Data \\ $\nu$} }}}
%     \put(-7, 14){\rotatebox{90}{\parbox{3cm}{\centering {Forward \\ $T^\star\# \eta$}}}}
%     \put(-7, -4){\rotatebox{90}{\parbox{3cm}{\centering {Backward \\$(T^\star)^{-1}\# \nu$}}}}
% \end{overpic}
% \caption{Transport experiments on two dimensional benchmarks using non-autonomous KODE. (Top row) The empirical samples from complex measure $\nu$. (Middle row) Samples generated after transporting isotropic gaussian measure $\eta$. (Bottom row) Samples generated from the backward flow of KODE on $\nu$.}
% \label{fig:toy_non_auto_results}
% % \vskip -0.2in
% \end{figure}

% Trajectory of non-auto transport 
\begin{figure*}[ht]
\centering
\begin{overpic}[width=0.75\textwidth, trim=20 30 25 25, clip]{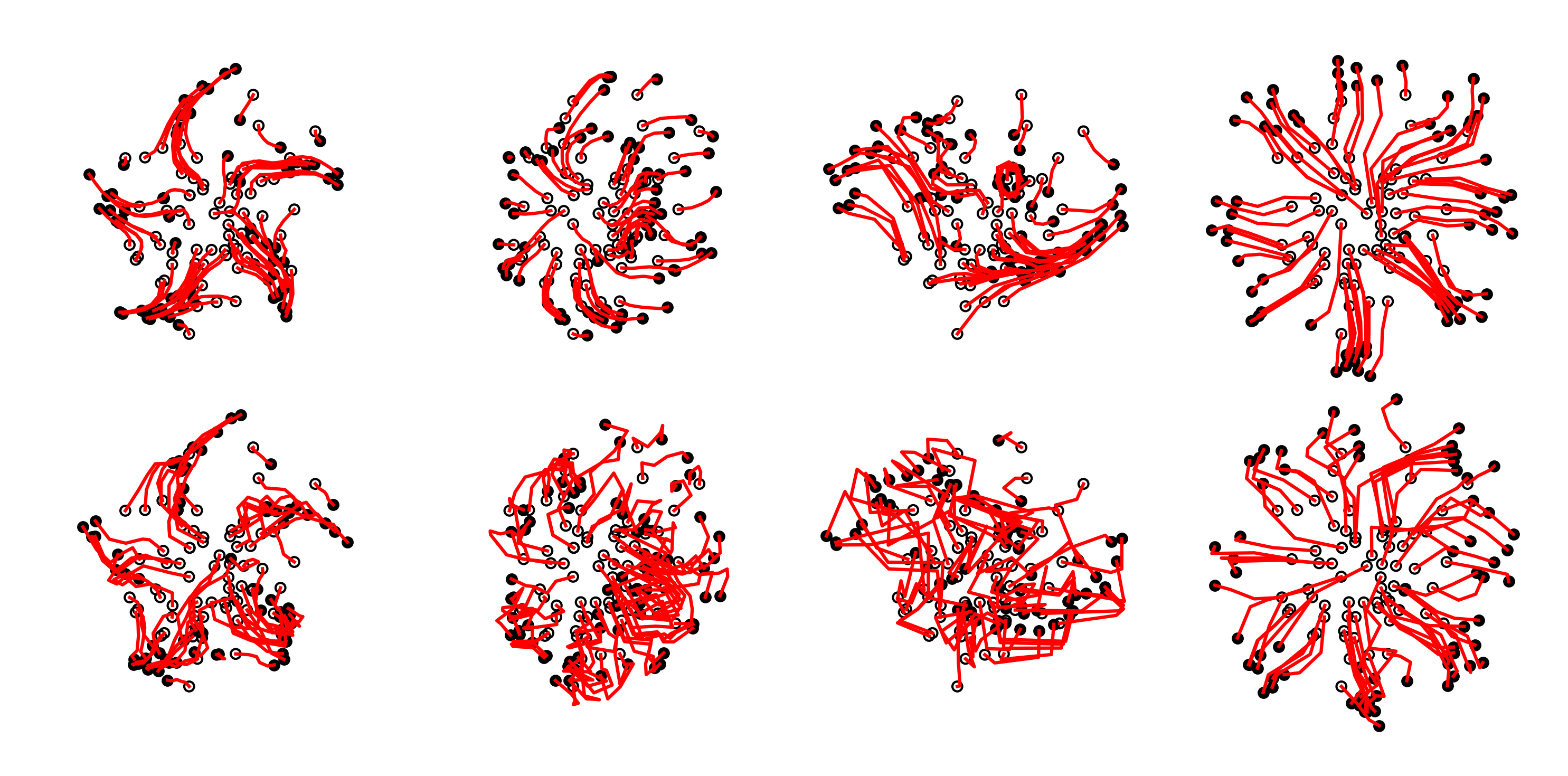}
    % \put(-8, 12.5){\rotatebox{90}{\parbox{3cm}{\centering \normalsize Forward flow }}}
    \put(-8, 22){\rotatebox{90}{\parbox{3cm}{\centering \small $\lambda_2 \neq 0$}}}
    \put(-8, -2){\rotatebox{90}{\parbox{3cm}{\centering \small $\lambda_2 = 0$}}}
\end{overpic}
\caption{KODE sample trajectories for  2D benchmarks 
with different choices of $\lambda_2$ in \eqref{sobolev-in-time-RKHS-norm}: (Top row) 
using $\lambda_2 \neq 0$ so the time derivatives of the coefficients 
are penalized; (Bottom row) using $\lambda_2 =0$.}
\label{fig:toy_non_auto_trajectory}
\vskip -0.2in
\end{figure*}

\subsection{Higher-dimensional benchmarks}\label{sec:high-dim-benchmark}
Next, we compared KODE with OT-Flow on high-dimensional data sets. \Cref{tab:numerical_results} shows our detailed quantitative comparisons 
akin to the 2D benchmarks. We found that non-autonomous KODE achieves competitive performance in three out of five examples. In particular, for the \texttt{POWER} data set, KODE significantly outperformed OT-Flow with a comparable training time while using a third of the parameters.  For the \texttt{GAS} and \texttt{BSDS300} data sets, KODE was on par with OT-FLOW albeit using  fewer parameters and faster training time in some cases. For \texttt{HEPMASS} and \texttt{MINIBOONE} KODE was not competitive, which is interesting 
since these data sets are lower-dimensional than {\texttt{BSDS300}}, suggesting 
that the latter data set may be high dimensional but somewhat simpler.
We also found that in all of these examples the non-autonomous KODE model outperformed 
autonomous KODE which is not surprising given that the former has more 
flexibility, but this observation implies that time-varying velocity fields 
do lead to better performance. In the case of the {\texttt{BSDS300}} data set 
it is interesting to note that autonomous KODE also achieved good performance, which 
further implies that this data set is not very complex.

% the performance lags by an order of magnitude to OT-Flow. Adding more parameters to the model did not improve performance. Between the KODE models, non-autonomous KODE outperforms autonomous KODE in all benchmarks, 

\Cref{fig:gas_time_dep_results} shows visualizations of the marginal distributions for the \texttt{GAS} data set. The top row shows 2D marginals from the target measure $\nu$, while the second row shows corresponding marginals for the KODE samples. Visually, these marginals 
are very close except that the KODE marginals are slightly blurred, indicating 
that the finer features of the data set are not captured.
The third row of \Cref{fig:gas_time_dep_results} shows the result of 
backward transport of test samples towards the Gaussian reference. Compared 
to the 2D benchmarks here we see that the normalizing flow is not as close to 
Gaussian as before. We found that this issue can be mitigated 
by adding an extra loss term during training which not only minimizers the MMD 
between the generate samples and the target, but also minimizes the MMD 
between the reference samples and the pull-back of the target samples. 
This additional penalty term leads to significantly better samples 
in the backward transport setting as indicated in the fourth row of \Cref{fig:gas_time_dep_results}.

% we see a good match between them. Third row shows the pullback measure on the reference $(T^*)^{-1}\# \nu$ obtained from the reverse flow of the learned ODE. In contrast to the two-dimensional case, the pullback measure does not resemble the reference isotropic gaussian measure $\eta.$ To improve the quality of the pullback, we train the ODE to transport $\eta$ to $\nu$ in the forward flow and $\nu$ to $\eta$ in the backward flow in an alternating fashion for every training iteration. The last row shows the resulting pullback measure. We see that the quality of the pullback improves significantly and resembles the gaussian reference $\eta$. The corresponding pushforward resembles the pushforward obtained by training only in the forward direction (result not shown). 

Finally, \Cref{fig:hepmass_time_dep_results} visualizes 2D marginals of the
\texttt{HEPMASS} data set which was one of the examples where KODE was not competitive 
with OT-Flow. Here we see that while the marginals are not a perfect match, KODE 
appears to capture the significant structural features of the target measure $\nu$.
We conjecture that this problem is difficult for KODE since the MMD between the 
reference and target samples is very small and our choices of the kernel 
$K$ are not sufficient.

% Non-autonomous KODE captures the general structure of the target measures; $\nu$ and $T^\star\#\eta$ are visually quite similar. However, the learned measure is more diffuse, resulting in smoother edges along the marginals, when compared to the target measure. This leads to MMD values that are approximately an order of magnitude lower than those obtained with the OTFlow method. Additional visualizations are shown in the supplementary materials.

% Result of non-auto transport for GAS
\begin{figure*}[ht]
\centering
\begin{overpic}[trim=5 5 5 5, clip, width=0.75\textwidth]{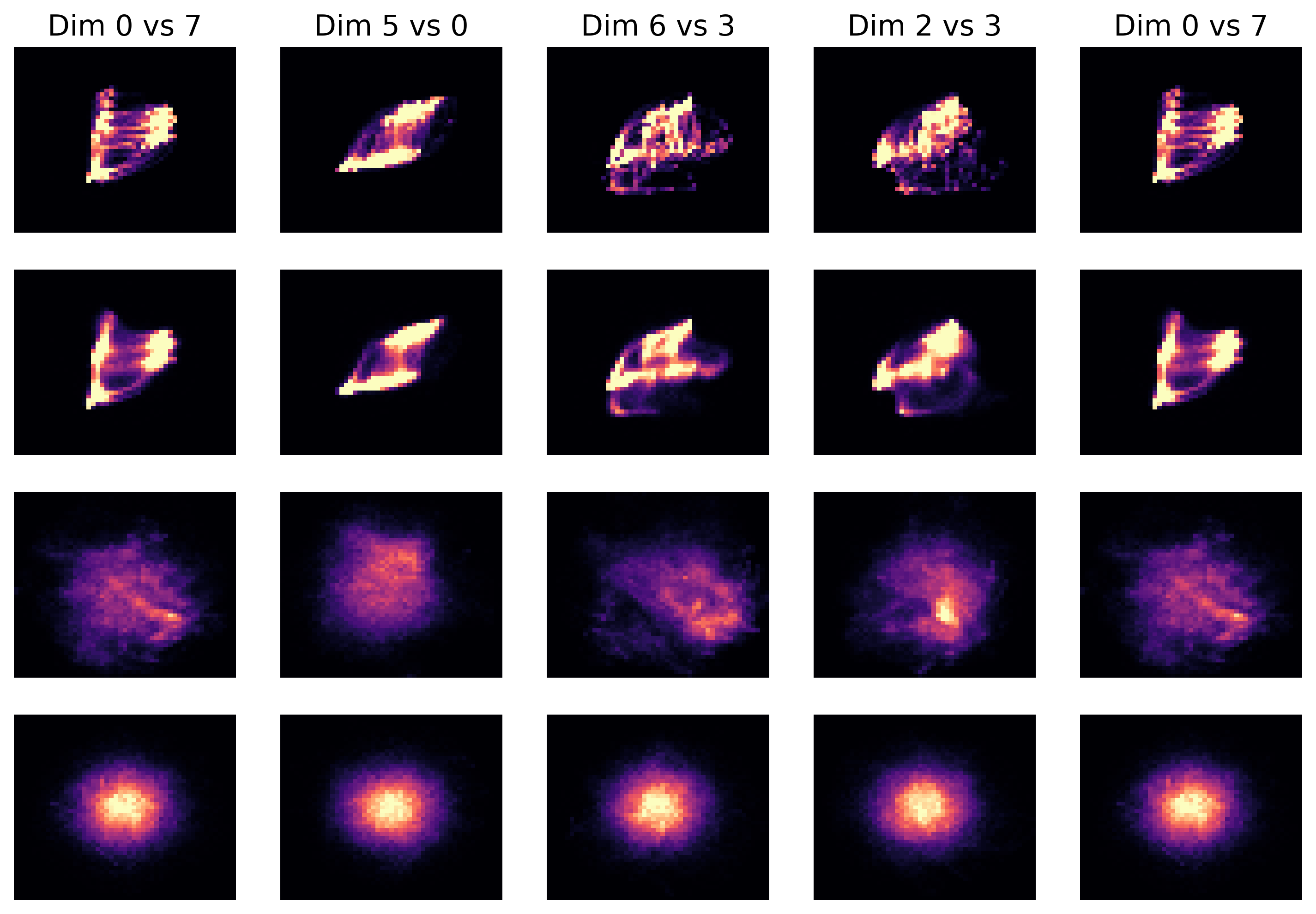}
    \put(-8, 47){\rotatebox{90}{\parbox{3cm}{\centering \small Data \\ $\nu$ }}}
    \put(-8, 30){\rotatebox{90}{\parbox{3cm}{\centering \small KODE \\ Forward}}}
    \put(-8, 16){\rotatebox{90}{\parbox{2cm}{\centering \small KODE \\ Backward }}}
    \put(-15, -2){\rotatebox{90}{\parbox{2cm}{\centering \small KODE \\ Backward (with extra 
    training)}}}
\end{overpic}
\caption{Transport experiments on \texttt{GAS} benchmark using non-autonomous KODE. (First row) Marginal distributions of the data measure $\nu$. (Second row) Marginal distributions of the learned pushforward measure. (Third row) Marginal distributions of the pullback measure obtained from the backward flow of KODE.
(Fourth row) Marginal distributions of the pullback measure obtained from the backward flow of KODE trained to transport $\eta$ and $\nu$ between each other simultaneously.}
\label{fig:gas_time_dep_results}
\end{figure*}
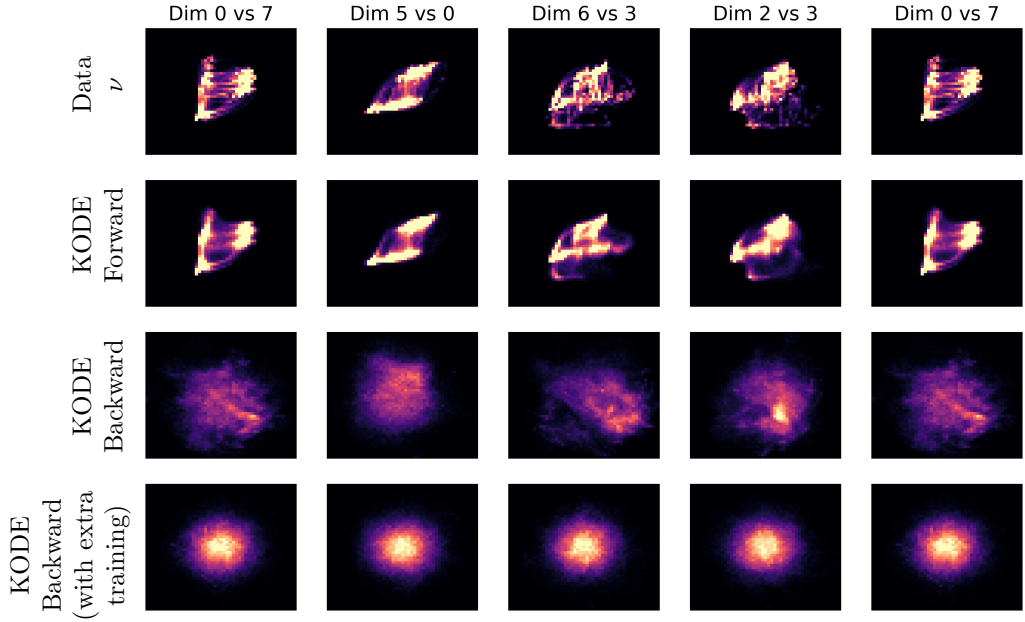

% Result of non-auto transport for HEPMASS
\begin{figure*}[ht]
\centering
\begin{overpic}[trim=5 235 5 5, clip, width=0.75\textwidth]             {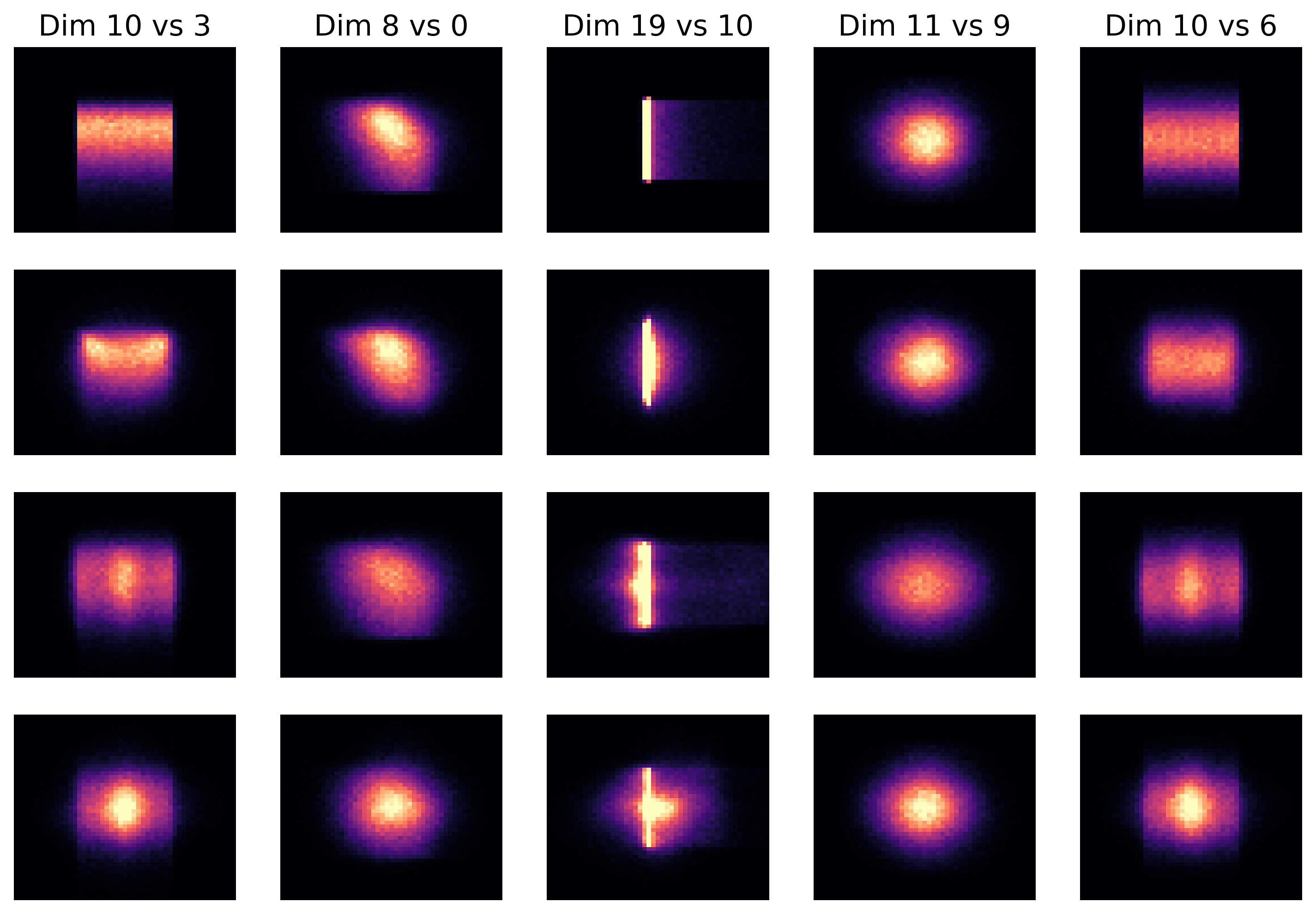}
    \put(-6, 13){\rotatebox{90}{\parbox{3cm}{\centering \small $\nu$ }}}
    \put(-6, -3){\rotatebox{90}{\parbox{3cm}{\centering \small KODE}}}
\end{overpic}
\caption{Transport experiments on \texttt{HEPMASS} benchmark using non-autonomous KODE. (First row) 2D marginals of the true data set; (Second row) 2D marginals of 
samples generated by KODE.}
\label{fig:hepmass_time_dep_results}
\end{figure*}

\subsubsection{Generative modeling for MNIST}\label{sec:MNIST}
Here we used KODE to  generate images akin to the  MNIST data set. 
Since applying KODE directly in the pixel space leads to non-satisfactory results 
we paired our algorithm with an intermediate autoencoder to reduce 
dimension of the data set. 
% As an intermediate step, we reduce the dimensionality of the dataset using an autoencoder architecture; this is because the KODE models did not scale to the original dimension of $784$.  
Consider an encoder $E:\R^{784} \rightarrow \R^d$ and a decoder $D:\R^{d} \rightarrow \R^{784}$ for MNIST such that $D(E(x)) \approx x$. 
If $\nu$ denotes the original target measure in the pixel space, we aim to
learn a map $T$ using KODE such that $T \# \eta$ is close to $E \# \nu$.  
Once the model is trained, we can generate a new image by drawing $z \sim \eta$, 
and evaluating $D\circ T(z)$. To train the autoencoder
we used a feedforward neural network with a dense layer for both the encoder and the decoder using ReLU activation while the output layers used a sigmoid activation function. 
We trained the autoencoder separately from KODE and made sure not to train to completion 
or to use a variational autoencoder model to ensure that KODE still had to 
generate samples from an interesting distribution \footnote{Variational 
autoencoders train $E$ in such a way that $E \# \nu$ is close to a Gaussian. In this case training KODE (or any other generative model) to transform $\eta$ to $E \# \nu$ is moot since an affine 
map would be sufficient.}.
\Cref{fig:mnist} shows the digits generated for $d=10$. The generated samples generally resemble handwritten digits, though we do see some malformed or implausible samples. We note that in this example simply passing $\eta$ through the decoder $D$ results in images 
that are mostly noise and so the trained KODE model is crucial to the generator.

% ; the map $T^\star$ is necessary for generating digits. In our experiments, larger values of $d$ resulted in worse samAples. We conjecture that this is due to the insensitivity of MMD loss as opposed to the expressivity of the KODE model class. With larger $d$, the target measure $E\#\nu$ increasingly resembles a gaussian; the vanilla MMD loss is unable to quantify the fine differences between $\eta$ and $E\#\nu$.

\begin{figure*}[ht]
    \centering
    \begin{overpic}[trim=0 0 0 35, clip, width=0.5 \textwidth]{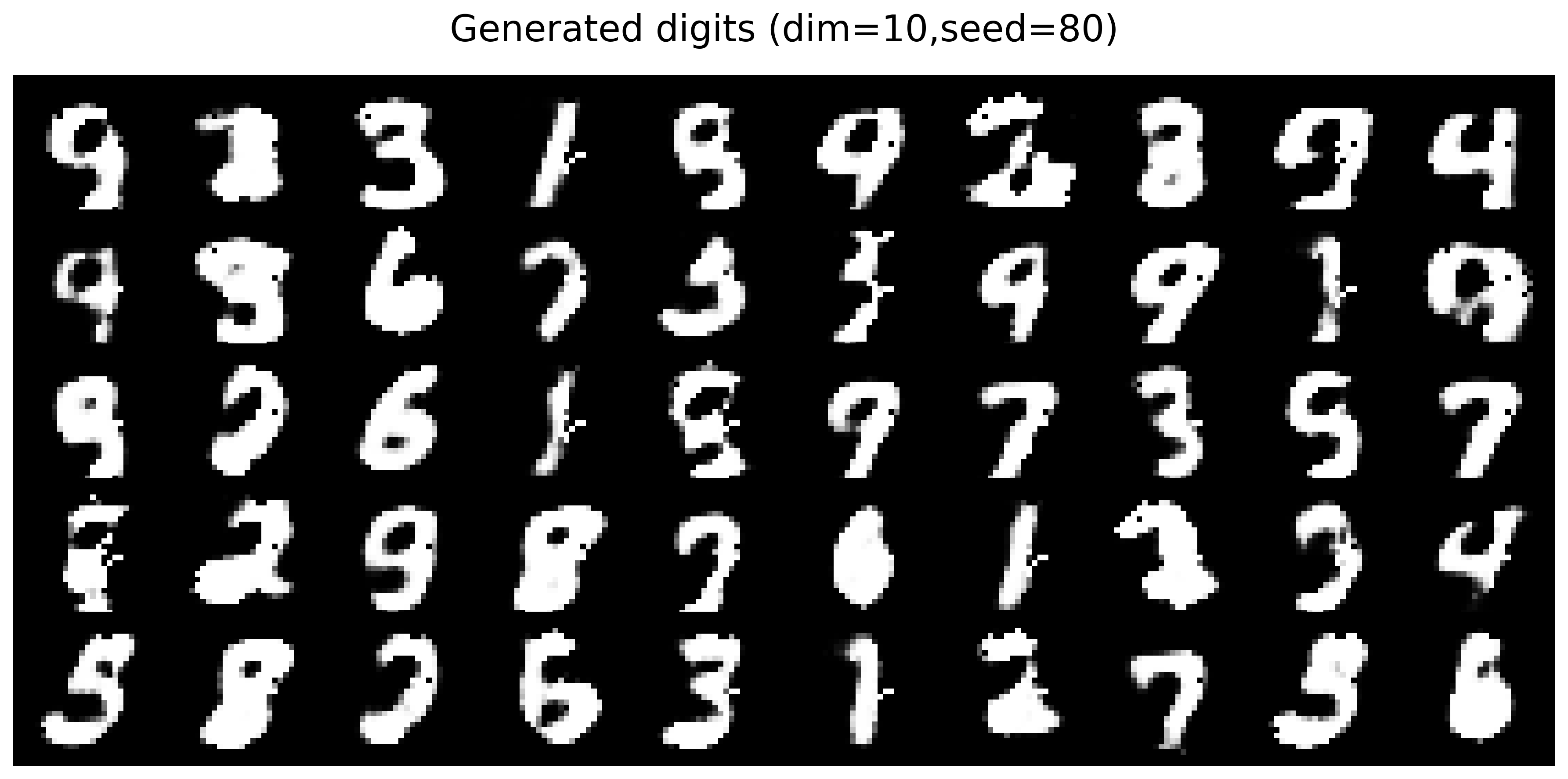}        
    \end{overpic}
    \caption{Generated MNIST digits using  KODE coupled with an autoencoder.}
    \label{fig:mnist}
\end{figure*}

\subsection{Triangular transport for conditioning}\label{sec:triangular-transport}
For our final set of experiments we demonstrate how 
KODE can be easily modified to perform likelihood-free and 
purely data-driven inference using the theory of 
triangular transport maps developed in
\cite{baptista2023conditional}. 
First, we briefly describe the mathematical framework of triangular 
maps and how KODE needs to be modified and then we present an 
example concerning parameter estimation in a
2D benchmarks and a Lotka-Volterra ODE model with four unknown 
parameters. 

Letting $\Y = \R^m$ and $\U = \R^d$ consider a target measure
$\nu \in \PP( \Y \times \U)$. Our goal here is to
obtain a generative model to sample from the
conditional measure $\nu(\cdot \mid y)$ for $y \in \Y$.
Letting $\nu_\Y$ denote the $\Y$-marginal of $\nu$ we
consider the reference measure
$\eta = \nu_\Y \otimes \eta_\U
\in \PP(\Y \times \U)$ with $\eta_\U \in \PP(\U)$ an
arbitrary measure. Finally we consider {\it triangular
  transport maps} of the form
\begin{equation*}
  T(y, u) = (y, T_\U( y, u) ), \quad
  T_\Y: \Y \to \Y,\quad T_\U: \Y \times \U \to \U.
\end{equation*}
Then for such triangular maps we have, by \cite[Thm.~2.4]{baptista2023conditional}, that
if $T \# \eta = \nu$  then $T_\U(y, \cdot) \# \eta_\U = \nu(\cdot \mid y)$.
This  result is the underlying principle for our modification of KODE in order to perform conditional simulation; see also the recent work
\cite{wang2023efficient}. More precisely, we wish to define the map $T$ as the flow of an ODE
but in such a way that the resulting transport map $T(y, u)
= \phi(1, (y,u))$ is of the triangular form. This can be easily achieved by simply putting the $\Y$-coordinates of the
velocity field $v$ in our formula to zero. To this end,
we obtain the following formulation which should be compared with \eqref{KODE-restricted-recalled}:
\begin{equation}
%v^{S,N}_\lambda  :=
\left\{
\begin{aligned}
      & \argmin_{v \in \VV} &&\MMD_K ( \phi(1, \cdot) \# \eta^N, \nu^N) + \lambda  \sum_{\ell=m+1}^{d+m} 
      \| v_\ell \|_{\V}^2 \\ 
      & \st &&  \phi_t  = v( t, \phi),  \quad \phi(0, x) = x \quad \forall x  \in \Omega, \\
      & && v(s) = (v_1(s), \dots, v_{d+ m}(s)) \\
      & &&  v_\ell = 0, \text{  for  } \ell =1, \dots, m  \\
      & &&  v_\ell = c_\ell^T V(S, \cdot),
      \text{  for  } \ell=m+1, \dots, d+m 
\end{aligned}
\right. \label{conditional-KODE}
\end{equation}
To this end, the modification of KODE for conditional simulation is nearly trivial. We note that in the setting where target samples
$\{ (y_j, u_j) \}_{j=1}^N \sim \nu$ are given we can easily generate
reference samples by forming the pairs $\{(y_j, \tilde{u_j}) \}_{j=1}^N
\sim \eta$ where the $y_j$'s are copied from the target samples
while the $\tilde{u}_j \sim \eta_\U$ are generated from the arbitrary
reference on $\U$. This approach was employed in order to produce
the results in \Cref{fig:conditioning_time_dep} by modifying the
code that was used for \Cref{fig:toy_non_auto_results}.
 
%\subsubsection{Triangular maps for conditioning}

%Let $\U, \W \subset \R^d$ and $\Y \subset \R^m$ be bounded open sets. Consider the probability measures $\nu \in \PP(\Y \times \U)$ and $\eta_\W \in \PP(\W)$. Let $\nu(\cdot|y)$ denote the conditional of $\nu$ on the $y$ variable. We want to find a map $T^\star_\U: \Y \times \W \rightarrow \U$ such that $T^\star_\U(y, \cdot)\# \eta_\W \approx \nu(\cdot|y)$. We consider triangular maps $T: \Y \times \W  \rightarrow \Y \times \U$ of the form  $T(y ,w) = (y, T_\U(y, w))$  for all $(y, w) \in \Y \times \W.$ Take reference measure $\eta \in \PP(\Y \times \W)$ of the form $ \eta = \nu_\Y \bigotimes \eta_\W$ where $\nu_\Y$ denotes the $\Y$-marginal of $\nu$.  
%For any map $T^\star \# \eta = \nu$, then it holds that $T^\star_\U(y, \cdot)\#\eta_\W = \nu(\cdot|y)$ \cite[Thm.~2.4]{MGAN}. Hence, we have reduced the problem of conditional transport to the original measure transport problem (\Cref{KODE-formulation-generic}).

%We consider a set of two dimensional benchmark examples shown in \Cref{fig:conditioning_time_dep}. Our goal is to sample vertical slices of the target measure $\nu$.  We use $J=5000$ target samples and take $\eta_\mathcal{\W} = N(0, 1).$ The results are presented in \Cref{fig:conditioning_time_dep}. We generally observe good agreement between the numerical and true conditional histograms across both examples with larger errors along slices that pass through low density regions.  

\subsubsection{Parameter Inference for a Lotka-Volterra ODE}
We will now use KODE for likelihood-free inference of the parameters
of a Lotka-Volterra ODE, which describes the population dynamics of two interacting species using a pair of first-order nonlinear ODEs;
this example was used in \cite{baptista2023conditional} as a
benchmark for a triangular GAN model for conditioning. 
The ODE has the form
\begin{align*}\label{LV model}
    \dfrac{dp_1}{dt} &= \alpha p_1(t) - \beta p_1(t) p_2(t) \\
    \dfrac{dp_1}{dt} &= -\gamma p_2(t) + \delta p_1(t)p_2(t)
\end{align*}
with initial condition $p(0) = (30, 1).$ Here $P_1, P_2$ denote
the populations of a prey and a predator respectively. The rate of change of two populations is driven by four parameters $u=(\alpha, \beta, \gamma, \delta) \in \R^4.$ Our goal here is to infer the values of these parameters
from noisy observations of the state of the ODE.

We consider the ground truth value of the parameters
$u^\dagger = ( 0.92, 0.05, 1.50, 0.02 )$ and simulate the ODE
up to time $T=20$. The state of the ODE is then observed at
time intervals of size $\Delta t = 2$, i.e.,
$y^\dagger_{k,i} =  p_i(k \Delta t) + \xi_{k,i}$ for $i=1,2$ and $k=1, \dots, 9$
and with log-normal observation noise $\log \xi_{k,i} \sim N(0, \gamma^2)$
for $\gamma = 0.01$. 

To infer the parameter $u$ we employ Bayes' rule with a 
standard normal 
prior on $u$.
% since the parameters need to be positive, i.e.,
% $\log u \sim \mathcal{N}(\mu, 0.5I_4)$ with
% $\mu = (-0.125,  -3, -0.125, -3)$ and 
We will use 
$\eta_\U$ to
denote this prior since it will also be used as our reference.
To generate samples from $\nu$ we proceed as follows:
First draw $u_j \sim \eta_\U$ and then numerically
solve the ODE with $u_j$  and simulate the data $y_j$ for $j=1,\dots, N$.
This procedure generates samples from $\nu$, the joint distribution of $u, y$
under our prior for $u$ and the model for the data $y$. Then Bayes'
rule states that $\nu( \cdot \mid y^\dagger)$ is precisely the posterior
distribution of $u$ given $y^\dagger$.

\Cref{fig:lotka-volterra-good} shows an example application of KODE
for our Lotka-Volterra model using $N= 10^5$ target and reference samples.
Here we compared the KODE-conditional samples with those generated
by an adaptive MCMC algorithm run with a large number of steps to ensure
convergence to the posterior and take it as the ``ground truth''.
We generally observed
good agreement between the two-dimensional marginals of the
KODE and MCMC samples, although KODE seems to underestimate variances.
The true parameter $u^\dagger$ that generated the data
is denoted using a black circle.

%(denoted in black) is contained in the bulk of the posterior distributions. Compared to MCMC, the KODE model tends to underestimate the posterior variance. 

%The goal of this task is to infer these parameters given noisy observations of the population of the two species aat select times. The populations $p(t) \in \R^2_+$ change according to the coupled ODEs

%For training data, we generate parameters $u$ from a log-normal prior distribution $\log u \sim \mathcal{N}(\mu, 0.5I_4))$ with parameter $\mu$. Using these parameters, we simulate the ODE for T time units and collect observations $y_k$ of the state every $\Delta t_{obs}$ time units. The observations $y_k$ are corrupted with an independent log-normal noise. We use $J=10^5$ samples  and take $\eta_\W=N(0, I_4)$. After training, we sample from the posterior density $\nu(u|y=y^*)$ using an MCMC algorithm, which we take to be the ground truth, and the KODE model. For more details of the task, please refer to $\cite{MGAN}$.

We show 100,000 posterior parameter samples from KODE i.e. $T_\U(y^*, w_i)$ for $w_i \sim \mathcal{N}(0, I_4)$ and from an adaptive Metropolis MCMC sampler in \Cref{fig:lotka-volterra-good}.  We observe similar one and two-dimensional marginal distributions using both methods. The true parameter $u^*$ that generated the data (denoted in black) is contained in the bulk of the posterior distributions. Compared to MCMC, the KODE model tends to underestimate the posterior variance. 

\begin{figure}[htbp] \label{fig:lotka-volterra-good}
    \vskip 0.2cm
    \centering
    % First minipage for the first image
    \begin{minipage}[b]{0.45\linewidth}
        \begin{overpic}[trim=0 0 0 25, clip, width=\linewidth]{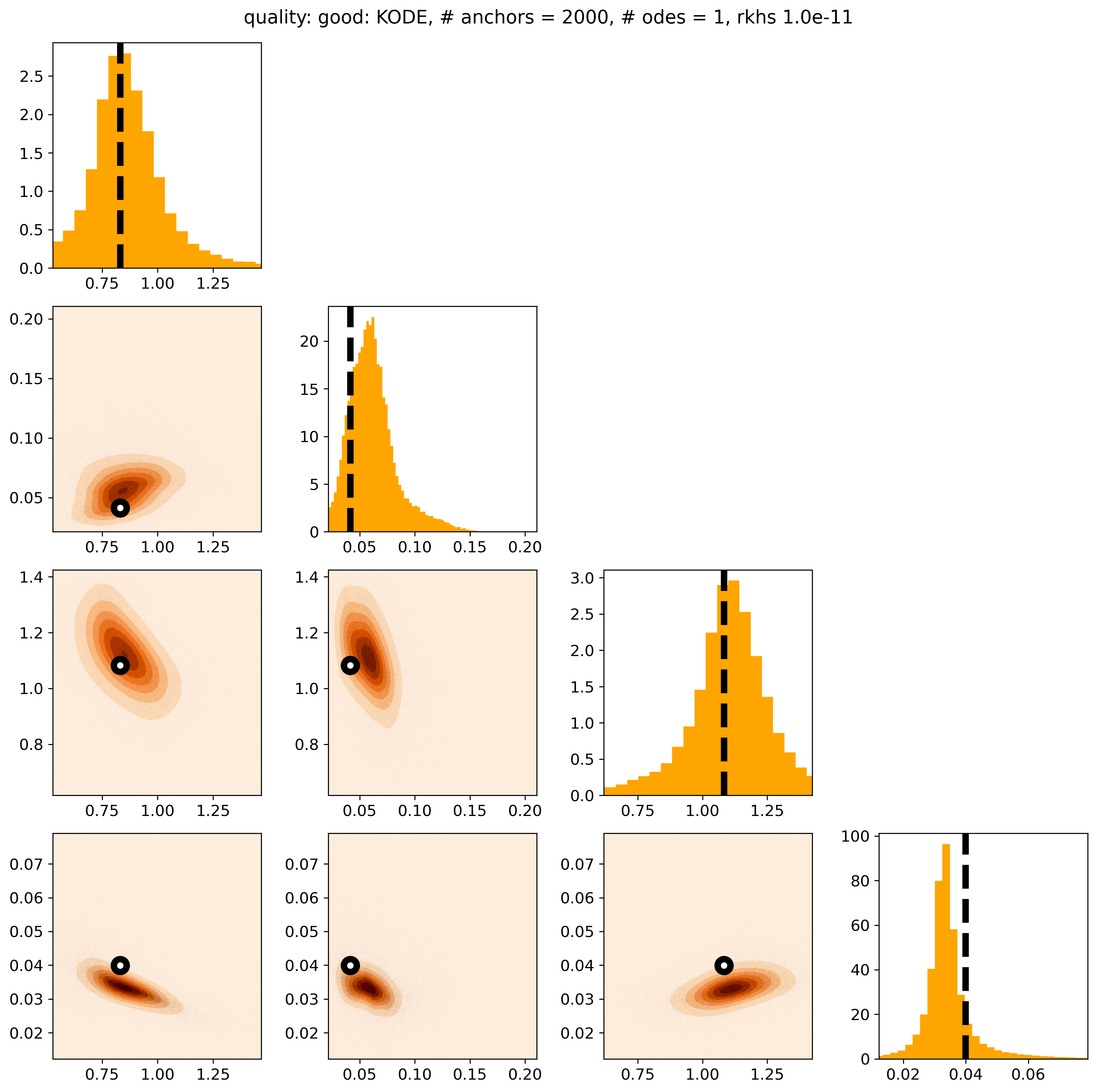}
            \put(15, -4){{\parbox{3cm}{\small $\alpha$}}}
            \put(40, -4){{\parbox{3cm}{\small $\beta$}}}
            \put(65, -4){{\parbox{3cm}{\small $\gamma$}}}
            \put(90, -4){{\parbox{3cm}{\small $\delta$}}}

            % vertical
            \put(-4, 82){{\parbox{3cm}{\small $\alpha$}}}
            \put(-4, 60){{\parbox{3cm}{\small $\beta$}}}
            \put(-4, 37.5){{\parbox{3cm}{\small $\gamma$}}}
            \put(-4, 12){{\parbox{3cm}{\small $\delta$}}}
        \end{overpic}
    \end{minipage}
    \hfill % Add some space between the images
    % Second minipage for the second image
    \begin{minipage}[b]{0.45\linewidth}
        \begin{overpic}[trim=0 0 0 25, clip, width=\linewidth]
        {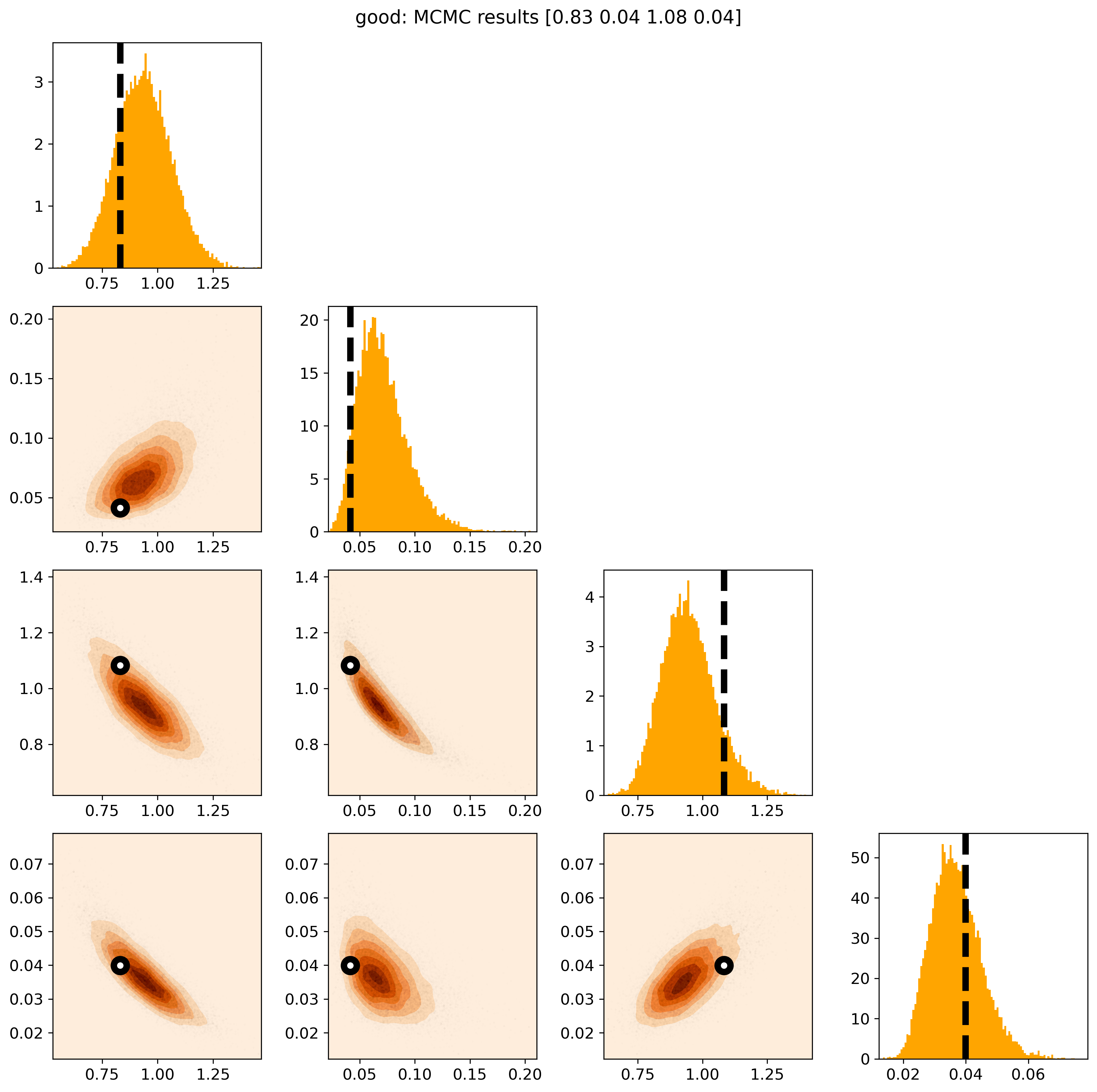}
            \put(15, -4){{\parbox{3cm}{\small $\alpha$}}}
            \put(40, -4){{\parbox{3cm}{\small $\beta$}}}
            \put(65, -4){{\parbox{3cm}{\small $\gamma$}}}
            \put(90, -4){{\parbox{3cm}{\small $\delta$}}}

            % vertical
            \put(-4, 82){{\parbox{3cm}{\small $\alpha$}}}
            \put(-4, 60){{\parbox{3cm}{\small $\beta$}}}
            \put(-4, 37.5){{\parbox{3cm}{\small $\gamma$}}}
            \put(-4, 12){{\parbox{3cm}{\small $\delta$}}}
        \end{overpic}
    \end{minipage}
    \caption{Samples from the posterior measure of the parameters
in the Lotka-Volterra model using (left) the triangular KODE model, and (right) the adaptive Metropolis MCMC algorithm.}
\end{figure}
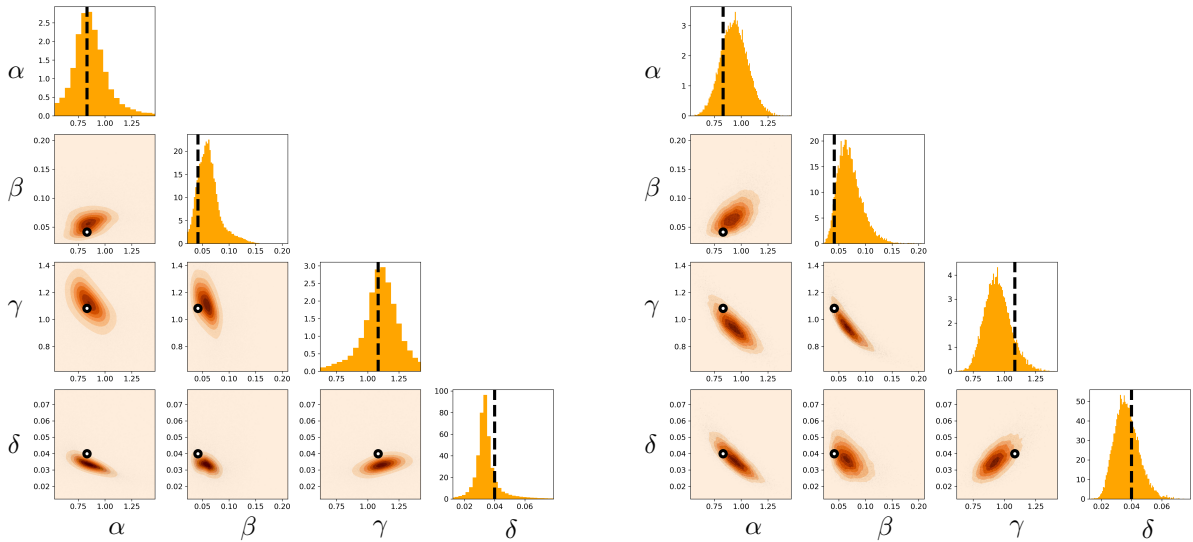

\section{Conclusion}
\label{sec:conclusion}
We introduced KODE, an approach for transport of
measures with a view towards generative modeling, that
was based on the theory of RKHSs and inspired by
the literature on diffeomorphic matching. We presented
a theoretical analysis of our model under idealized assumptions
leading to quantitative error bounds in terms of the
number of samples in the training data (i.e., sample complexity)
as well as the complexity/degrees of freedom of our model
(i.e., approximation error). To our knowlege, our theory is
one few results of this kind that clearly shows the interaction
between approximation error of the model and sample complexity.

We further developed algorithms based on the KODE framework that
are simple and convenient to implement and mathematically interpretable. We demonstrated the effectiveness of the method on a
number of benchmark transport problems in low- to high-dimensional settings and showed that KODE performs well on these
benchmarks, outperforming neural net methods of a similar
size in some of the benchmarks. There were also instances where
KODE was not competitive.

Our results open the door to various avenues of research
in both theory and algorithms. From a theory standpoint, it would be interesting to close the gap between our theoretical framework and
the implemented version of KODE: (1) it is interesting to
further characterize the relationship between our constrained
formulation of KODE \eqref{KODE-restricted-recalled} and its unconstrained version \eqref{KODE-unconstrained} which is implemented;  (2) it is interesting to
extend our error bounds to account for the error of
numerical ODE solver that is used in the model; (3) finally, it
is interesting to consider KODE with discrepancies besides
MMD and try to obtain error bounds in that case. From an algorithmic standpoint,
it is interesting to investigate how the performance
of KODE can be improved: (1) our current implementation is
very sensitive to the choice of the inducing points which may
also scale badly with dimension. Therefore it is interesting to
investigate other strategies such as random features;
(2) our formulation is also sensitive to the choice of the
kernel in the MMD term and so a strategy for choosing that
kernel to maximally differentiate the target and generated
samples is of interest; (3) while we currently employ
stochastic gradient descent for the training of KODE it is
interesting to design Newton-type algorithms that can
leverage our RKHS penalty terms to achieve fast convergence.

%We conclude that our method is pretty competitive  not consistently the best when employing vanilla kernels but it comes with simple and transparent theoretical guarantees. There are several possible innovations possible in the future. 

% comparison of effect of regularization

% show that the trajectories are reversible

% qq plot for 2d non-auto
% qq plot for 2d auto

% qq plot for gas auto
% qq plot for gas non-auto

% qq plot for Hepmass auto 
% qq plot for Hepmass non-auto

\section*{Acknowledgments}
BP  and BH were supported by the NSF grant DMS-208535. B.P. was supported by NSF Graduate Research Fellowship Program under Grant No. DGE-1762114.. 
PB, BH, and HO acknowledge support from the Air Force Office of Scientific Research under MURI award number FA9550-20-1-0358 (Machine Learning and Physics-Based Modeling and Simulation).

\bibliographystyle{siamplain}
\bibliography{references}

\end{document}

%% file: shared.tex
\usepackage{lipsum}
\usepackage{amsfonts}
\usepackage{graphicx}
\usepackage{epstopdf}
\ifpdf
  \DeclareGraphicsExtensions{.eps,.pdf,.png,.jpg}
\else
  \DeclareGraphicsExtensions{.eps}
\fi

\usepackage{algpseudocode}
\algrenewcommand\algorithmicrequire{\textbf{Input:}}
\algrenewcommand\algorithmicensure{\textbf{Output:}}

\usepackage{enumitem}
\setlist[enumerate]{leftmargin=.5in}
\setlist[itemize]{leftmargin=.5in}

\usepackage[colorinlistoftodos,prependcaption,textsize=tiny]{todonotes}

\newsiamremark{remark}{Remark}
\newsiamremark{hypothesis}{Hypothesis}
\crefname{hypothesis}{Hypothesis}{Hypotheses}
\newsiamthm{claim}{Claim}

\headers{}{B. Pandey, B. Hosseini, P. Batlle, H. Owhadi}

\title{Diffeomorphic  Measure Matching with Kernels for Generative Modeling
\thanks{Submitted to the editors on \today{}.
\funding{
BP and BH were supported by the NSF Grant DMS-2208535. BP was 
also supported by the NSF Graduate Research Fellowship Program DGE-1762114. 
PB, BH, and HO were supported by the Air Force Office of Scientific Research under MURI award number FA9550-20-1-0358 (Machine Learning and Physics-Based Modeling and Simulation). 
}}}

\author{Biraj Pandey\thanks{Department of Applied Mathematics, University of Washington, Seattle, WA, USA (\email{bpandey@uw.edu}, \email{bamdadh@uw.edu})}
\and
Bamdad Hosseini\footnotemark[2]
\and 
Pau Batlle\thanks{Department of Computing and Mathematical Sciences, California Institute of Technology, Pasadena, CA, USA 
  (\email{pbatllef@caltech.edu}, \email{owhadi@caltech.edu}).}
\and
Houman Owhadi\footnotemark[3]
}

\newcommand{\mcl}{\mathcal}

\newcommand{\mbb}{\mathbb}

\newcommand{\dd}{{\rm d}}

\newcommand{\Y}{\mcl{Y}}

\newcommand{\U}{\mcl{U}}

\newcommand{\V}{{\mcl{V}}}
\newcommand{\W}{{\mcl{W}}}
\newcommand{\VV}{\mathbb{V}}
\newcommand{\LL}{\mathbb{L}}

\newcommand{\T}{\mcl{T}}

\newcommand{\K}{\mathcal{K}}

\newcommand{\Q}{\mathcal{Q}}

\newcommand{\PP}{\mbb{P}}

\definecolor{darkred}{rgb}{.7,0,0}

\definecolor{darkgreen}{rgb}{.15,.55,0}

\definecolor{darkblue}{rgb}{0,0,0.7}

\usepackage{amsopn}

\DeclareMathOperator*{\argmin}{arg\,min}

\DeclareMathOperator*{\minimize}{minimize}

\DeclareMathOperator*{\st}{s.t.}